\documentclass[12pt]{article}

\usepackage[utf8]{inputenc}     
\usepackage[T1]{fontenc}        
\usepackage[margin=1in]{geometry}
\usepackage{newtxtext}
\usepackage{natbib}
\usepackage{url}                
\usepackage{hyperref}           
\hypersetup{
    colorlinks=true,
    linkcolor=blue,
    citecolor=blue,
    urlcolor=blue,
}

\usepackage{amsmath}
\usepackage{amsthm}
\usepackage{amsfonts}
\usepackage{amssymb}
\usepackage{mathtools}
\usepackage{bm}
\usepackage{graphicx}
\usepackage{booktabs}           
\usepackage{nicefrac}           
\usepackage{microtype}          
\usepackage{centernot}
\usepackage{mathrsfs}
\usepackage{comment}
\usepackage{float}
\usepackage{longtable}
\usepackage{tablefootnote}
\usepackage{standalone}
\usepackage[sans]{dsfont}
\usepackage{xspace}
\usepackage{multirow}
\usepackage{subfigure}
\usepackage{makecell}
\usepackage{array}
\usepackage[ruled]{algorithm2e}
\usepackage{verbatim}
\usepackage{enumerate}

\restylefloat{table}
\mathtoolsset{showonlyrefs}

\let\hat\widehat
\let\tilde\widetilde

\newcommand{\bX}{\bm{X}}

\newcommand{\EE}{\mathbb{E}}

\newcommand{\RR}{\mathbb{R}}

\newcommand{\bx}{\boldsymbol{x}}

\newcommand{\Rom}[1]{\text{\uppercase\expandafter{\romannumeral #1\relax}}}

\newcommand{\norm}[1]{\| #1\|}
\newcommand{\normx}[1]{\left\|#1\right\|}

\newcommand{\innerx}[2]{\left\langle #1,#2 \right\rangle}
\newcommand{\rbr}[1]{\left(#1\right)}
\newcommand{\sbr}[1]{\left[#1\right]}

\newtheorem{theorem}{Theorem}[section]
\newtheorem{lemma}[theorem]{Lemma}

\newtheorem{corollary}[theorem]{Corollary}

\newtheorem{assumption}{Assumption}

\newtheorem{remark}[theorem]{Remark}

\title{Stochastic Gradient Descent for Nonparametric Additive Regression}

\author{Xin Chen\thanks{xc5557@princeton.edu} \ and \ Jason M. Klusowski\thanks{jason.klusowski@princeton.edu} \\[0.5em]
\small Department of Operations Research and Financial Engineering, Princeton University}

\date{}

\begin{document}

\maketitle

\begin{abstract}
This paper introduces an iterative algorithm for training nonparametric additive models that enjoys favorable memory storage and computational requirements. The algorithm can be viewed as the functional counterpart of stochastic gradient descent, applied to the coefficients of a truncated basis expansion of the component functions. We show that the resulting estimator satisfies an oracle inequality that allows for model mis-specification. In the well-specified setting, by choosing the learning rate carefully across three distinct stages of training, we demonstrate that its risk is minimax optimal in terms of the dependence on both the dimensionality of the data and the size of the training sample. Unlike past work, we also provide polynomial convergence rates even when the covariates do not have full support on their domain.
\end{abstract}


\section{Introduction}

Suppose we obtain $n$ samples, each denoted as $(\bX_i, Y_i)$, where  $\bX_i = \big(X_i^{(1)},X_i^{(2)},\dots, X_i^{(p)}\big)^T  \in \RR^p$ is a $p$-dimensional covariate vector, and $Y_i \in \RR$ is a response value. In the general setting, a nonparametric regression model takes the form
\begin{equation}\label{nonpar}
\begin{aligned}
    &Y_i = f(\bX_i) + \varepsilon_i,
\end{aligned}
\end{equation}
where $\varepsilon_i $ is an unobserved noise variable. The goal is to derive an estimate for the regression function $f$ based on the available data.

\subsection{Nonparametric Additive Models}
It is well-known that when the regression function $f$ belongs to certain broad function classes, such as the H\"older smoothness class, the minimax optimal convergence rate for estimating $f$ is adversely affected by the dimension $p$ (see, e.g., \cite{gyorfi2002distribution}). To address this issue, it is typical to assume a specific functional form on the regression function, which reduces flexibility, but avoids the aforementioned negative effects on estimation. One widely adopted form is the nonparametric additive model \citep{580a479e-8bd3-310b-b7a7-84887bec0fed, hastie1990generalized}, which inherits the structural simplicity and interpretability characteristic of linear models, but without the rigidity of being parametric. Formally, the regression function $f$ admits the additive representation
\begin{align}\label{additive}
    f(\bX_i) = \alpha +  \sum_{k=1}^pf_k\left(X^{(k)}_i\right),
\end{align}
where $\alpha \in \RR$ is a constant, and $f_1, f_2,\dots, f_p$ is a collection of $p$ univariate component functions. Throughout this paper, we focus on the models \eqref{nonpar} and \eqref{additive}. In order to establish identifiability for the additive model, we assume throughout that each component function $f_k $ is centered (with respect to Lebesgue measure) in the sense that $\int f_k(x) dx = 0$ for each $k=1,2,\dots,p$.

There have been a number of influential papers dedicated to fitting the nonparametric additive model \eqref{additive}. For example, the backfitting algorithm was proposed by \cite{97272c9f-415d-38e1-b5c1-b5216f1a728e}, with its properties subsequently studied by \cite{7b56331a-3c32-333b-a3a9-80650cf8a6e6}. The local scoring backfitting method was introduced by \cite{hastie1990generalized}, and a local likelihood estimator was proposed by \cite{27f0e47c-7e6a-3f96-87cd-35773773c7ae}. These algorithms are commonly referred to as \textit{batch} methods, given that they entail fitting the model on the entire data set at once. Consequently, such methods demand substantial memory and computational resources for data analysis, posing a significant challenge when dealing with streaming or large-scale data sets. 

Recently, \cite{doi:10.1080/01621459.2023.2182213} adapted the smooth backfitting method \citep{40bc9d37-04e6-331a-92b0-fcf494a7785c, 29938989-35ba-3682-81ec-8d6d28e262dc} for use in an online environment, where data is processed in a streaming manner and estimators are updated with each new data point. This method, however, involves resolving a set of nonlinear integral equations through a dual iteration process and requires the dynamic selection of bandwidths for each component function from a group of candidates. These ingredients could potentially lead to time-intensive computations, particularly when the dimension is large (see Section \ref{subsec:univariate} for further details). Furthermore, the relationship between the convergence rate of their method and the dimension $p$ was not made explicit.

\subsection{Stochastic Gradient Descent}

Stochastic Gradient Descent (SGD) is a cornerstone optimization algorithm in machine learning, well known for its computational efficiency and its ability to enhance generalization, particularly in complex tasks such as training deep neural networks \citep{Goodfellow-et-al-2016}. Theoretical properties of SGD have been extensively studied in the context of linear regression \citep{jain2018parallelizing, ge2019step}, including in overparameterized settings \citep{zou2021benefits, zou2021benign, wu2022last}. More recently, SGD-based methods have attracted considerable attention in nonparametric regression, where the regression function is assumed to belong to a Reproducing Kernel Hilbert Space (RKHS) and estimators are constructed via kernel methods \citep{4406a3ce-1e59-3ed8-8c25-9f7d75b81f14, 6842642, article}. \citet{zhang2025learning} also investigates the minimax optimality of kernel-based SGD under certain model misspecifications. Despite their popularity, kernel-based SGD methods often suffer from substantial computational and memory burdens (see Section~\ref{subsec:univariate} for further details).

To alleviate some of these difficulties, \cite{zhang2022sieve} assumed the regression function lives in a Sobolev ellipsoid (also a RKHS), and circumvented the use of kernels by instead learning the coefficients of an orthogonal basis expansion of $ f $. They proposed an interesting online scheme, \emph{Sieve-SGD}, and showcased its minimal memory storage requirements (up to a logarithmic factor), low computational complexity, and optimal theoretical performance (in terms of the sample size). In a subsequent work, \citet{zhang2023online} further investigated parameter selection, providing examples in the context of Sieve-SGD. Although \citet{zhang2022sieve} also proposed an extension of Sieve-SGD to nonparametric additive models, they did not offer explicit guidance on parameter choices, and the corresponding theoretical guarantees were not established.

\subsection{Contributions}

In this paper, we propose an online estimator based on SGD in the context of nonparametric additive regression.  Our estimator can be regarded as the functional counterpart of  stochastic gradient descent, applied to the coefficients of a truncated basis expansion of the component functions. To reflect this operational methodology, we call our estimator the \textit{functional stochastic gradient descent (F-SGD)} estimator.
We demonstrate that F-SGD exhibits optimal theoretical performance with favorable memory storage requirements and  computational complexity. The contributions of this paper can be summarized as follows:
\begin{itemize}

\item On the theoretical front, we establish the minimax optimality of F-SGD under the nonparametric additive model. Specifically, when the regression function $f$  lies in the space of the sums of $p$ univariate Sobolev ellipsoids (denoted as $\mathcal{H}_p(s)$ in Section \ref{subsec:space}), we show that the Mean Squared Error (MSE) achieved by F-SGD is minimax optimal in terms of \emph{both} the dimensionality $p$ and the sample size $n$. In contrast, prior work has only demonstrated optimality with respect to $n$ \citep{zhang2022sieve,4406a3ce-1e59-3ed8-8c25-9f7d75b81f14, xue2022dynamic, doi:10.1080/01621459.2023.2182213, quan2024optimal}. 

More generally, when $f$ belongs to a space of sums of $p$ square-integrable univariate functions (denoted as $\mathcal{F}_p$ in Section \ref{subsec:space}), we prove that F-SGD satisfies an oracle inequality that accommodates model mis-specification. The resulting MSE bound consists of two components: The approximation error of $f$ using functions from $\mathcal{H}_p(s)$ and the estimation error within $\mathcal{H}_p(s)$ itself. Our analysis is built on a recursive inequality for the MSE derived from the SGD update rule, which proves significantly more tractable than previous SGD-based approaches \citep{zhang2022sieve,4406a3ce-1e59-3ed8-8c25-9f7d75b81f14}. 

In addition, we extend prior results by relaxing a common assumption that requires the input distribution to have full support over its domain (see detailed discussion in Section \ref{sec:asssumption}). Our analysis demonstrates that F-SGD achieves polynomial convergence rates even when the input density is zero in certain regions.

\item On the algorithmic front, our proposed method closely resembles conventional 
SGD, making it substantially simpler and more straightforward than online 
smooth backfitting \citep{yang2023online}. Moreover, when the dimension $p$ is large, F-SGD achieves 
significant improvements in both memory and computational complexity compared 
to online smooth backfitting.  Unlike past kernel SGD methods for nonparametric regression \citep{zhang2022sieve,4406a3ce-1e59-3ed8-8c25-9f7d75b81f14}, F-SGD streamlines the training process by excluding the usual Polyak averaging step and, in contrast to Sieve-SGD \citep{zhang2022sieve}, eliminating the need for a distinct learning rate per basis function, which may depend on unknown parameters.  When $f$ belongs to a Sobolev ellipsoid, F-SGD exhibits near-optimal space complexity similar to Sieve-SGD, but with an improvement by a logarithmic factor. The computational cost is also improved by a logarithmic factor. In comparison to kernel SGD methods, both space and computational complexity are improved by polynomial factors (see detailed discussion in Section \ref{subsec:univariate}). Moreover, thanks to its simplicity, F-SGD holds the potential for extensions, such as an online version of Lepski's method for adaptation to unknown smoothness (though a formal theoretical investigation of this extension is beyond the scope of this paper).

\end{itemize}

\noindent\textbf{Organization.} The remainder of the paper is organized as follows. In Section \ref{sec:pre}, we define some function spaces and establish several fundamental assumptions necessary for our model. In Section \ref{subsec:F-SGD}, we introduce our proposed F-SGD estimator. In Section \ref{sec:main}, we study its theoretical performance and elucidate its advantages through comparisons with contemporaneous methods. In Section \ref{sec:numerical}, we conduct a few numerical studies. Extensions of F-SGD are discussed in Section \ref{sec:discuss}. Proofs of the main results are provided in Section \ref{appdx:proof}, except for Theorem \ref{thm:polynomial}, which appears in the Supplementary Material \citep{chen2025supplement}.

\noindent\textbf{Notation.} Throughout th tis paper, we assume the pair $(\bX,Y)$ has joint distribution $P_{\bX,Y}$, and $\bX$ has marginal distribution $P_{\bX}$. For functions $f,g: \RR^p \to \RR$,  we let
\begin{align*}
    \|f\|^2 \coloneqq \int f^2(\bX)dP_{\bX}, \qquad \langle f,g\rangle \coloneqq  \int f(\bX)g(\bX)dP_{\bX}.
\end{align*}
We denote $a_n = O(b_n)$ if there exists a constant $C>0$ independent of $n$ such that $a_n \leq Cb_n$ for all $n$. The notation $a_n = \Theta(b_n)$ means $a_n = O(b_n)$ and $b_n = O(a_n)$. We use $\lfloor x \rfloor$ to represent the largest integer not exceeding $x$ and $\lceil x \rceil$ to denote the smallest integer not falling below $x$. For a vector $\bX \in \RR^p$, we use $X^{(k)}$ to denote the $k$-th component of $\bX$. The $\|\cdot\|_\infty$ norm for a continuous function $f$ is defined as $\|f\|_\infty \coloneqq \sup_{x \in \mathcal{X}}f(x)$ where $\mathcal{X}$ is the compact domain of $f$. 
\section{Preliminaries}\label{sec:pre}

In this section, we begin by defining some useful function spaces and then describing several fundamental assumptions necessary for our theory.

\subsection{Function Spaces} \label{subsec:space}
We denote $\mathcal{L}^2$ as the space of square-integrable univariate functions. To streamline our notation, we  refrain from explicitly specifying its domain at this point. We will later assume the domain is the unit interval $[0,1]$. We say a collection of functions $\{\psi_j: j=1,2,\dots\} \subset \mathcal{L}^2$ is an orthonormal basis of $\mathcal{L}^2$ (with respect to the Lebesgue measure) if 
$$
\int \psi_i(x)\psi_j(x)dx = \delta_{ij}
$$
where $\delta_{ij}$ is the Kronecker delta. Additionally, we say 
\begin{enumerate}
    \item[(i)] the orthonormal basis $\{\psi_j\}$ is centered if 
    $$
    \int \psi_j(x)dx = 0, \quad j=1,2,\dots
    $$
    \item[(ii)] the orthonormal basis $\{\psi_j\}$ is complete if for any $f \in \mathcal{L}^2$ there exists a unique sequence $\{\theta_{f,j}\}_{j=1}^\infty \in \ell^2$ such that
\begin{align*}
    f = \sum_{j=1}^\infty \theta_{f,j}\psi_j,
\end{align*}
where $\ell^2$ is the space of square convergent series.
\end{enumerate}
In $\mathcal{L}^2[0,1]$, i.e., the square-integrable function space over the interval $[0,1]$, a well-known example of a complete orthonormal basis is the trigonometric basis where $\psi_0(x) = 1 $, $\psi_{2k-1}(x) = \sqrt{2}\sin(2\pi kx)$, and $ \psi_{2k}(x) = \sqrt{2}\cos(2\pi kx)$ for $k=1,2,\dots$. By excluding the constant basis function $\psi_0$, the remaining trigonometric basis elements together form a centered orthonormal basis. Another frequently-used  complete orthonormal basis is the wavelet basis \citep{DBLP:books/daglib/0035708}.

Given an orthonormal basis $\{\psi_j\}$ (which may not necessarily be complete), we define $\mathcal{F}_1(\{\psi_j\})$ as a set of univariate square-integrable functions that possess expansions with respect to the basis $\{\psi_j\}$. In other words,
$$
\mathcal{F}_1(\{\psi_j\}) \coloneqq \Bigg\{f = \sum_{j=1}^\infty \theta_{f,j}\psi_j : \sum_{j=1}^\infty \theta_{f,j}^2 < \infty\Bigg\}.
$$
It is worth noting that when $\{\psi_j\}$ is complete, $\mathcal{F}_1(\{\psi_j\})$ is equivalent to $\mathcal{L}^2$.  Furthermore, we define the Sobolev space $\mathcal{H}_1(s, \{\psi_j\})$ as
$$
\mathcal{H}_1(s, \{\psi_j\}) \coloneqq \Bigg\{f = \sum_{j=1}^\infty \theta_{f,j}\psi_j : \sum_{j=1}^\infty (j^s\theta_{f,j})^2 < \infty\Bigg\},
$$
and also the Sobolev ellipsoid $\mathcal{W}_1(s,Q,\{\psi_j\})$ as:
$$
\mathcal{W}_1(s,Q, \{\psi_j\}) \coloneqq \Bigg\{f = \sum_{j=1}^\infty \theta_{f,j}\psi_j : \sum_{j=1}^\infty (j^s\theta_{f,j})^2 < Q^2\Bigg\}.
$$
It is easy to see that $\mathcal{W}_1(s,Q, \{\psi_j\}) \subset \mathcal{H}_1(s, \{\psi_j\})$. For $f \in \mathcal{H}_1(s, \{\psi_j\})$, we define the Sobolev norm as  $\|f\|_s^2 \coloneqq \sum_{j=1}^\infty (j^s\theta_{f,j})^2$. In the Sobolev ellipsoid, the parameter $s$ characterizes the smoothness of the function class. At this point, we assume that $s$ is known in advance. Unless otherwise specified, we only require $s > 0$, and it can be any real number. We will discuss the scenario in which $s$ is unknown in Section \ref{subsec:Lepski}. Sobolev spaces and Sobolev ellipsoids are widely explored spaces of interest, particularly in the context of batch nonparametric regression  \citep{gyorfi2002distribution,DBLP:books/daglib/0035708}. 
We consider the following multivariate (additive) extensions of these univariate function classes:
\begin{align*}
    &\mathcal{F}_p(\{\psi_{j1}\},\{\psi_{j2}\},\dots,\{\psi_{jp}\})\\
    &\coloneqq~ \bigg\{f: \RR^p \to \RR : f = \alpha +\sum_{k=1}^pf_k, \text{ where } \alpha\in \RR,\: f_k \in \mathcal{F}_1(\{\psi_{jk}\}),\: k=1,2,\dots,p    \bigg\},
\end{align*}
where $\{\psi_{j1}\},\{\psi_{j2}\},\dots,\{\psi_{jp}\}$ are $p$ centered orthonormal bases provided in advance. When these bases are chosen to be trigonometric, excluding the constant basis function, the class \\$\mathcal{F}_p(\{\psi_{j1}\},\{\psi_{j2}\},\dots,\{\psi_{jp}\})$ would be equivalent to the sum of $p$ separate $\mathcal{L}^2$ spaces. Specifically,
\begin{align*}
    \mathcal{F}_p(\{\psi_{j1}\},\{\psi_{j2}\},\dots,\{\psi_{jp}\}) = \left\{f = f_1+f_2+\dots+f_p : f_k \in \mathcal{L}^2, \text{ where } k=1,2,\dots,p  \right\}.
\end{align*}
In the same manner, we can define $\mathcal{H}_p(s,\{\psi_{j1}\},\{\psi_{j2}\},\dots,\{\psi_{jp}\})$ and \\ $\mathcal{W}_p(s,Q,\{\psi_{j1}\},\{\psi_{j2}\},\dots,\{\psi_{jp}\})$ by specifying that the component functions $f_k $ belong to \\ $ \mathcal{H}_1(s, \{\psi_{jk}\})$ and $ \mathcal{W}_1(s,Q, \{\psi_{jk}\})$, respectively, for each $k=1,2,\dots,p$. For ease of exposition, we restrict ourselves to the setting where each component function $f_k$ has the same smoothness parameter $s$, but all of the forthcoming theory can be generalized to accommodate heterogeneous smoothness across the component functions. When it is clear which basis we are using, we will denote  $\mathcal{F}_p(\{\psi_{j1}\},\{\psi_{j2}\},\dots,\{\psi_{jp}\})$ simply by $\mathcal{F}_p$. The same applies to $\mathcal{H}_p(s)$ and $\mathcal{W}_p(s,Q)$.

\subsection{Assumptions}\label{sec:asssumption}
Having introduced the necessary function spaces, we now proceed to outline several key assumptions.

\begin{assumption}\label{asmp:1}
    The data $(\bX_i, Y_i)_{i \in \mathbb{N}} \in [0, 1]^p \times \mathbb{R}$ are drawn i.i.d. from a distribution $P_{\bX,Y}$. 
\end{assumption} 

\begin{assumption}\label{asmp:2} 
The marginal distribution of $\bX$, denoted by $P_{\bX}$, is absolutely continuous with respect to the Lebesgue measure on $[0, 1]^p$. Let $p_{\bX}$ denote its probability density function. For some constants $C_1, C_2$ independent of $p$, such that $0<C_1<C_2<\infty$,
\begin{align*}
    C_1 \leq p_{\bX}(\bx) \leq C_2 \quad \text{for all }\bx \in [0, 1]^p.
\end{align*}
\end{assumption}

\setcounter{assumption}{1} 
\renewcommand{\theassumption}{2$'$} 
\begin{assumption}\label{asmp:2'}
    For some constant $C_2$ independent of $p$, such that $0 < C_2 < \infty$,
\begin{align*}
  p_{\bX}(\bx) \leq C_2 \quad \text{for all }\bx \in [0, 1]^p.
\end{align*}
\end{assumption}

\renewcommand{\theassumption}{\arabic{assumption}}
\begin{assumption}\label{asmp:3}
     Let $\{\psi_{j1}\},\{\psi_{j2}\},\dots,\{\psi_{jp}\}$ be $p$ centered and uniformly bounded orthonormal bases on $[0,1]$, where $\|\psi_{jk}\|_\infty \leq M $, $M\geq 1$,  for each  $k=1,2,\dots,p$ and $j=1,2,\dots$. The regression function $f$ belongs to $ \mathcal{F}_p(\{\psi_{j1}\},\{\psi_{j2}\},\dots,\{\psi_{jp}\})$.
\end{assumption}

\begin{assumption}\label{asmp:4}
    The noise $\varepsilon$ satisfies $ \EE[\varepsilon \mid \bX] = 0 $ and $ \EE[\varepsilon^2 \mid \bX] \leq \sigma^2 $, almost surely, for some $ \sigma > 0 $. 
\end{assumption}
 
We make several comments on the assumptions. 
Assumption \ref{asmp:2} imposes upper and lower bounds on the joint density of $\bX$. This condition is commonly assumed in other online methods; for example, kernel SGD (Assumption 2 in \cite{4406a3ce-1e59-3ed8-8c25-9f7d75b81f14}), Sieve-SGD (Assumption 2 in \cite{zhang2022sieve}), and online smooth backfitting (Assumption 1 in \cite{doi:10.1080/01621459.2023.2182213}). However, in large dimensional settings with dependent covariates, it can be difficult to ensure that the density is uniformly bounded away from zero across all of $[0,1]^p$; in fact, the density may vanish on certain regions of the domain. To reflect this, we also consider a relaxed version, stated as Assumption \ref{asmp:2'}, which does not require a uniform lower bound. 

Although the stronger Assumption \ref{asmp:2} is required to establish optimal convergence rates, our simulationsâ€”similar to those in \cite{zhang2022sieve}â€”demonstrate that F-SGD continues to achieve near-optimal rates even when the support of $p_{\bX}$ is strictly smaller than $[0,1]^p$. Moreover, in Section \ref{subsec:univariate}, we provide a formal result showing that F-SGD achieves polynomial convergence rates even when the input density is not bounded away from zeroâ€”that is, under Assumption \ref{asmp:2'}.

Assumption \ref{asmp:3} is similar to that of \cite{zhang2022sieve}, though we only require $f$ to be in $\mathcal{F}_p$ rather than in a Sobolev ellipsoid. Note that the boundedness assumption is satisfied for the trigonometric basis described in Section \ref{subsec:space} with $ M = \sqrt{2} $.
Compared to \cite{doi:10.1080/01621459.2023.2182213}, who also considered an additive model, we do not require each component function to be twice continuously differentiable. In Section \ref{subsec:additive}, we will establish an oracle inequality within the hypothesis space $\mathcal{F}_p$. This inequality will recover the optimal rate when we restrict $f$ to be within the the space of the sum of $p$ univariate Sobolev ellipsoids. 

Assumption \ref{asmp:4} on the noise distribution is also more relaxed compared to previous literature. For example, \cite{zhang2022sieve} assume that either $\varepsilon$ is independent of $\bX$ and has finite second moment, or is almost surely bounded. \cite{doi:10.1080/01621459.2023.2182213} imposed higher-order moment conditions on $\varepsilon$, in addition to assuming the conditional variance $\text{Var}(Y \mid \bX) = \EE[\varepsilon^2 \mid \bX]$ is strictly positive and twice continuously differentiable with respect to $\bX$. In contrast, our assumption only involves boundedness of the conditional variance of $Y$ given $\bX$.

Like \cite{zhang2022sieve}, we also point out that our assumptions,  in comparison to other kernel methods for nonparametric estimation \citep{4406a3ce-1e59-3ed8-8c25-9f7d75b81f14, 6842642,article}, are more straightforward. The assumptions in those papers typically involve verifying certain conditions related to the spectrum of the RKHS covariance operator. This task can be challenging due to the involvement of the unknown distribution $P_{\bX}$ in the covariance operator. We encourage readers to see \cite{zhang2022sieve} for a nice discussion on these differences.

\section{Functional Stochastic Gradient Descent}\label{subsec:F-SGD}

In this section, we formulate an estimator based on SGD for the   nonparametric additive model \eqref{additive}. Throughout this section, we grant Assumption \ref{asmp:3}, i.e., 
$$f\in \mathcal{F}_p(\{\psi_{j1}\},\{\psi_{j2}\},\dots,\{\psi_{jp}\}).$$ 
We proceed under the assumption that the set of centered basis functions $\{\psi_{jk}\}$, for each $k=1,2,\dots,p$, is pre-specified. The population level objective is to solve
\begin{align*}
    \min_{f \in \mathcal{F}_p}\frac{1}{2} \EE \left[\left( Y - f(\bX)
 \right)^2\right] = \min_{\alpha,\, \boldsymbol{\beta}}  L(\alpha, \boldsymbol{\beta}),
\end{align*}
where
\begin{align*}
     L(\alpha, \boldsymbol{\beta}) \coloneqq \frac{1}{2}\EE \left[\left( Y - \alpha - \sum_{k=1}^p\sum_{j=1}^{\infty}\beta_{jk}\psi_{jk}\left(X^{(k)} \right)
 \right)^2\right],
\end{align*}
using a gradient-based optimization algorithm.
Here $$\boldsymbol{\beta} = (\beta_{11}, \beta_{21}, \dots \beta_{12}, \beta_{22}, \dots \beta_{1p}, \beta_{2p}, \dots)^T$$ represents the infinite-dimensional parameter vector. At each iteration step $ i $,
we can compute the (infinite-dimensional) gradient to obtain the gradient descent updates:
\begin{align*}
&\boldsymbol{\beta}_i = \boldsymbol{\beta}_{i-1} - \gamma_i \nabla_{\boldsymbol{\beta}}L(\alpha, \boldsymbol{\beta})\vert_{\alpha = \alpha_{i-1},\; \boldsymbol{\beta}=\boldsymbol{\beta}_{i-1}},\\
& \alpha_i = \alpha_{i-1} - \gamma_i\nabla_{\alpha}L(\alpha, \boldsymbol{\beta})\vert_{\alpha = \alpha_{i-1},\; \boldsymbol{\beta}=\boldsymbol{\beta}_{i-1}}.
\end{align*}
This can be translated into an update rule in the function space $\mathcal{F}_p$ by setting $ \tilde{f}_i(\bX) = \alpha_i + \boldsymbol{\psi}(\bX)^T\boldsymbol{\beta}_i $, where $$\boldsymbol{\psi}(\bX) = \big(\psi_{11}\big(X^{(1)}\big), \psi_{21}\big(X^{(1)}\big), \dots \psi_{12}\big(X^{(2)}\big), \psi_{22}\big(X^{(2)}\big), \dots \psi_{1p}\big(X^{(p)}\big), \psi_{2p}\big(X^{(p)}\big), \dots\big)^T$$ is the corresponding infinite-dimensional basis vector. 

We obtain the functional update:
\begin{align}
\tilde{f}_i(\bX) = \tilde{f}_{i-1}(\bX) - \gamma_i \left( \nabla_{\alpha}L(\alpha, \boldsymbol{\beta})\vert_{\alpha = \alpha_{i-1},\; \boldsymbol{\beta}=\boldsymbol{\beta}_{i-1}} + 
 \boldsymbol{\psi}(\bX)^T \nabla_{\boldsymbol{\beta}}L(\alpha, \boldsymbol{\beta})\vert_{\alpha = \alpha_{i-1},\; \boldsymbol{\beta}=\boldsymbol{\beta}_{i-1}}\right). \label{functional update}
\end{align}
In practice, computing $\nabla_{\alpha}L(\alpha, \boldsymbol{\beta})$ and  $\nabla_{\boldsymbol{\beta}}L(\alpha, \boldsymbol{\beta})$ directly is infeasible, as they are unknown population quantities. Stochastic gradient descent uses the latest data point for estimating these gradients.
That is, unbiased estimates of the partial derivatives take the form: 
\begin{align*}
    &\hat{\nabla}_{\alpha} L(\alpha, \boldsymbol{\beta})\vert_{\alpha = \alpha_{i-1},\; \boldsymbol{\beta}=\boldsymbol{\beta}_{i-1}} = -\left(Y_i-\tilde{f}_{i-1}(\bX_i)\right),\\
    &\hat{\nabla}_{\boldsymbol{\beta}} L(\alpha, \boldsymbol{\beta})\vert_{\alpha = \alpha_{i-1},\; \boldsymbol{\beta}=\boldsymbol{\beta}_{i-1}} = -\left(Y_i-\tilde{f}_{i-1}(\bX_i)\right)\boldsymbol{\psi}(\bX_i).
\end{align*}
Plugging these estimates into the update \eqref{functional update}, we obtain the functional update
\begin{align*}
    \tilde f_i(\bX) = \tilde f_{i-1}(\bX) + \gamma_i\left(Y_i - \tilde f_{i-1}\left(\bX_i\right)\right) \left(1+\sum_{k=1}^p\sum_{j=1}^{\infty}\psi_{jk}\left(X_i^{(k)}\right)\psi_{jk}\left(\bX\right)\right),
\end{align*}
with initialization $\tilde f_0 = 0$. However, the convergence of $\tilde f_i$ is not guaranteed even if all the basis functions $\psi_{jk}$ are bounded. To address this issue, inspired by the usual projection estimator for Sobolev spaces \citep{DBLP:books/daglib/0035708}, we project $\tilde f_i$ onto a finite-dimensional space, namely, the linear span of $ \{\psi_0, \, \psi_{jk} : j = 1, 2, \dots, J_i \; \text{and} \; k = 1, 2, \dots, p \} $, $ J_i \in \mathbb{N} $ and $ \psi_0 \equiv 1$, resulting in the following recursion:
\begin{align}\label{update:l2loss}
    &\hat{f}_i(\bX) = \hat{f}_{i-1}(\bX) + \gamma_i\left(Y_i - \hat{f}_{i-1}\left(\bX_i\right)\right) \left(1+\sum_{k=1}^p\sum_{j=1}^{J_i}\psi_{jk}\left(X_i^{(k)}\right)\psi_{jk}\left(\bX\right)\right),
\end{align}
with  initialization $\hat f_0 = 0$. We refer to $\hat f_i$ as the \textit{functional stochastic gradient descent (F-SGD)} estimator of  $f$. Here, $\gamma_i$ denotes the learning rate and $J_i$ is the dimension of the projection. Both parameters control the bias-variance tradeoff and must be tuned accordingly. The choices of $\gamma_i$ and $J_i$ will be elaborated upon in Section \ref{subsec:additive}. Additionally, in Section \ref{subsec:additive} and \ref{subsec:univariate}, we will demonstrate that this estimator achieves optimal statistical performance, while maintaining modest computational and memory requirements.

Recently, \cite{zhang2022sieve} were also motivated by SGD, aptly naming their procedure \emph{sieve stochastic gradient descent (Sieve-SGD)}. In the additive model context, they suggested the update rule: 
\begin{align}
   &  \tilde{f}_i(\bX) = \frac{i}{i+1}\tilde{f}_{i-1}(\bX) + \frac{1}{i+1} \tilde g_i(\bX) \label{zhang1}\\
    & \tilde g_i(\bX) = \tilde g_{i-1}(\bX) + \gamma_i\left(Y_i - \tilde g_{i-1}\left(\bX_i\right)\right)\left(1+\sum_{k=1}^p\sum_{j=1}^{J_{ik}}t_{jk}\psi_{jk}\left(X_i^{(k)}\right)\psi_{jk}\left(\bX\right)\right), \label{zhang2}
\end{align}
with initialization $\tilde g_0 = \tilde f_0 = 0$,  where $ t_{jk} $ is a (non-adaptive) component-specific learning rate that may involve information on the  smoothness level $ s $, e.g., $ t_{j} = j^{-2\omega} $, where $ 1/2 < \omega < s $.
By comparing \eqref{update:l2loss} with \eqref{zhang1} and \eqref{zhang2}, we see that the component-specific learning rate $ t_{jk} $ and the Polyak averaging step \eqref{zhang1} are both omitted with F-SGD. 
Due to the less common nature of a non-adaptive, component-specific learning rate in standard SGD implementations, our F-SGD estimator more closely mirrors the characteristics of vanilla SGD. The additional simplicity of our estimator makes it not only more straightforward to implement but also facilitates a more nuanced theoretical analysis, enabling us to understand the impact of both the dimensionality and sample size on performance. We conduct a comprehensive comparison between our F-SGD estimator and several prior estimators, including Sieve-SGD, in Section \ref{subsec:univariate}.

\section{Main Results}\label{sec:main}
In this section, we provide theoretical guarantees for F-SGD \eqref{update:l2loss}.
In Section \ref{subsec:additive}, we establish an oracle inequality when $f \in \mathcal{F}_p$ and recover the minimax optimal rate when $f \in \mathcal{W}_p$. We also derive polynomial convergence rates when the input density is not bounded away from zero. In Section \ref{subsec:univariate}, we compare F-SGD with extant methods and elucidate its advantages.

\subsection{Oracle Inequality}\label{subsec:additive}
We begin by presenting an oracle inequality that depends on both the dimensionality $p$ and the sample size $n$, precisely characterizing the trade-off in estimating $f$. This inequality balances the capability of candidate component functions $g_k$ to approximate the $f_k$ against their respective Sobolev norms, $\|g_k\|_s$. 

\begin{theorem}\label{theorem:l2loss}
Suppose Assumptions \ref{asmp:1}-\ref{asmp:4} hold. Furthermore, assume $n \geq C_0 p^{1+1/(2s)}$ where $C_0 > 1$ is a constant. Let $A_1$, $ A_2 $, and $ B$ be constants such that $A_1 = (2s+1)A_2 $, $ A_2 \geq \frac{2}{C_1}$, and $B \leq \frac{1}{4C_2M^2A_2^2}$. Assume $p \geq \frac{1}{B^{2s}}$.
Set $\gamma_i$ and $J_i$ according to three stages of training:
\begin{enumerate}
    \item[i] When $1\leq i\leq \frac{p}{B}$, set $\gamma_i = J_i = 0$.
    \item[ii] When $\frac{p}{B} < i\leq p^{1+1/(2s)}$, set $\gamma_i = \frac{A_1}{i}$ and $J_i = \lceil \frac{Bi}{p}\rceil$.
    \item[iii] When $p^{1+1/(2s)} < i \leq n$, set $\gamma_i = \frac{A_2}{i}$ and $J_i = \lceil Bi^{\frac{1}{2s+1}}\rceil$.
\end{enumerate}
Then the MSE of F-SGD \eqref{update:l2loss}, initialized with $\hat f_0 = 0$, satisfies
\begin{align}\label{thm_ineq1}
    \EE \left[\norm{\hat{f}_n - f}^2 \right] \leq C\inf_{g\in \mathcal{H}_p(s)} \Bigg\{ \|g-f\|^2 + \left(p\sigma^2+\|f\|^2+\sum_{k=1}^p\|g_k\|^2_s \right) n^{-\frac{2s}{2s+1}}\Bigg\}.
\end{align}
Here $g_1,g_2,\dots,g_p$ represent the components functions of $g$, as $g(\bX) = \tau + \sum_{k=1}^pg_k(X^{(k)})$  for $\tau \in \RR$, and  $C = C(C_0,C_1,C_2,M,s)$ is a constant. 
\end{theorem}

We obtain an interesting corollary of Theorem \ref{theorem:l2loss} when $f$ belongs to $\mathcal{W}_p(s, Q)  \subset \mathcal{H}_p(s)$. Here, we can set $g = f$ in inequality \eqref{thm_ineq1}. Note that $\|f\|^2 \leq C_2(\alpha^2 + pQ^2)$, and hence we obtain the following inequality:
\begin{align*}
     \EE \left[\norm{\hat{f}_n - f}^2 \right] \leq C(C_2\alpha^2+p\sigma^2+2C_2pQ^2)n^{-\frac{2s}{2s+1}}.
\end{align*}
Given that the minimax lower bound for estimating regression functions in $ \mathcal{W}_p(s, Q)$ is $pn^{-2s/(2s+1)}$ \citep{raskutti2009lower}, this inequality implies that F-SGD achieves minimax-optimal convergence in terms of both $ n $ and $ p $. We present this result in the following corollary.
\begin{corollary}\label{cor:W}
    Suppose the conditions in Theorem \ref{theorem:l2loss} hold. Furthermore, suppose $f \in \mathcal{W}_p(s, Q)$. Then the MSE of F-SGD \eqref{update:l2loss}, initialized with $\hat f_0 = 0$, satisfies
    \begin{align*}
        \EE \left[ \norm{\hat{f}_n - f}^2 \right] = O\left(pn^{-\frac{2s}{2s+1}}\right).
    \end{align*}
\end{corollary}
The detailed proof of Theorem~\ref{theorem:l2loss} is provided in Section \ref{appdx:proof}. Unlike the delicate analysis for Sieve-SGD \citep{zhang2022sieve} and other kernel SGD methods \citep{4406a3ce-1e59-3ed8-8c25-9f7d75b81f14, 6842642}, the proof of Theorem \ref{thm_ineq1} is notably less involved, necessitating only the derivation of a simple recursive inequality for the MSE based on the update rule \eqref{update:l2loss}, without the need to handle the signal and noise terms separately, and without any advanced functional analysis. The recursion turns out to be curiously similar to those involved in the analysis of the Robbinsâ€“Monro stochastic approximation algorithm, a precursor to SGD, e.g., \citep[Lemma 1]{chung1954stochastic}.

We outline the proof at a high level here. 
At the $i$-th step, the recursive inequality is obtained by applying the update rule~\eqref{update:l2loss} 
and expanding the error $\|\hat{f}_i - f\|^2$. 
The key step is bounding the cross term, which we formalize as a lemma in Section \ref{appdx:proof}. This bound depends on two quantities: $\mathbb{E}[\|\hat{f}_{i-1} - f\|^2]$ and the $\ell_2$ error under the uniform measure 
$\mathbb{E}\left[ \int \big( \hat{f}_{i-1}(\bx) - f(\bx) \big)^2 d\bx \right]$. 
For each training stage, we repeatedly solve the recursive inequality using another lemma given in Section \ref{appdx:proof}.

\begin{remark}
To illustrate how the total error in estimating an additive model can be decomposed into the cumulative error of estimating each individual component function and the constant term, consider the special case where $\bX \sim \text{Unif}([0,1]^p)$. In this setting, orthogonality among the basis functions allows us to write:
\begin{align*}
    &\inf_{g\in \mathcal{H}_p(s)} \Bigg\{ \|g-f\|^2 + \left(p\sigma^2+\|f\|^2+\sum_{k=1}^p\|g_k\|^2_s \right) n^{-\frac{2s}{2s+1}}\Bigg\}\\
    &=~ \alpha^2n^{-\frac{2s}{2s+1}}+\sum_{k=1}^p\inf_{g_k\in \mathcal{H}_1(s)} \Bigg\{ \|g_k-f_k\|^2 + \left(\sigma^2+\|f_k\|^2+\|g_k\|^2_s \right) n^{-\frac{2s}{2s+1}}\Bigg\}.
\end{align*}
\end{remark}
\begin{remark}
In Theorem \ref{theorem:l2loss}, we determine the values of $\gamma_i$ and $J_i$ across three distinct stages of training. In the first stage, we do not update $\hat f_i$, and the current MSE stays constant at  its initial value of $\|f\|^2$. By the end of stage ii, at the $p^{1+1/(2s)}$-th step, we employ $p^{1/(2s)}$ basis functions to estimate each component function, for a total of $p^{1+1/(2s)}$ basis functions for estimating $f$. When $f$ is well-specified (i.e., $f \in \mathcal{W}_p$), this results in the current MSE being a constant multiple of $\|f\|^2/p$. Moving onto stage iii, we update approximately $i^{1/(2s+1)}$ basis functions for each component function at each step, which results in a minimax optimal MSE. For a single component function, the number $i^{1/(2s+1)}$ is known to be nearly space-optimal \citep{zhang2022sieve,quan2024optimal}. As we need to store the coefficients of basis functions for $p$ component functions, we anticipate that F-SGD, with a  $\Theta\big(pi^{1/(2s+1)}\big)$-sized space expense, is also (nearly) space-optimal. Further details on the storage requirements of F-SGD will be explored in Section \ref{subsec:univariate}.
\end{remark}

If achieving optimal convergence with respect to $ p $ is not a priorityâ€”such as when $ p $ is smallâ€”we can omit the first two stages (i and ii) and focus directly on the third stage. Importantly, in this case, we can choose each $\gamma_i$ to be independent of $s$. The theoretical performance for this version of F-SGD is formalized in the next theorem. As will be explored in Section \ref{subsec:Lepski}, the fact that $J_i$ is now the only tuning parameter that depends on $ s $ also allows us to develop an online version of Lepski's method.

\begin{theorem}\label{thm:fixed}
    Suppose Assumptions \ref{asmp:1}-\ref{asmp:4} hold.  Set $\gamma_i = \frac{A}{i}$ and $J_i = \lfloor Bi^{\frac{1}{2s+1}} \rfloor$, where $A$ and $B$ are constants independent of $ s $ and satisfying $A \geq \frac{2}{C_1}$ and $B \leq \frac{1}{2pC_2M^2A^2}$. Then the MSE of F-SGD \eqref{update:l2loss}, initialized with $\hat f_0 = 0$, satisfies
\begin{align*}
    \EE \left[\norm{\hat{f}_n - f}^2 \right] \leq C\inf_{g\in \mathcal{H}_p(s)} \Bigg\{ \|g-f\|^2 + \left(\sigma^2+\|f\|^2+\sum_{k=1}^p\|g_k\|^2_s \right) n^{-\frac{2s}{2s+1}}\Bigg\}.
\end{align*}
Here $g_1,g_2,\dots,g_p$ represents the components functions of $g$, as $g(\bX) = \tau + \sum_{k=1}^pg_k(X^{(k)})$ for $\tau \in \RR$, and  $C = C(p,C_1,C_2,M,s)$ is a   constant.  
\end{theorem}

We briefly discuss the selection of constants. We acknowledge that choosing appropriate constants $A$ and $B$ is generally nontrivial, as also noted in \cite{zhang2022sieve,zhang2023online}. While Theorem~\ref{thm:fixed} provides sufficient conditions for theoretical guarantees, empirical performance tends to be less sensitive. For instance, when $\boldsymbol{X} \sim \mathrm{Unif}([0,1]^p)$, the algorithm appears to work well for a wide range of $A$ and $B$. More broadly, in higher-dimensional settings, we recommend using a larger $A$ and a smaller $B$ to stabilize updates and mitigate overfitting. We also note that \cite{zhang2023online} recently proposed a weighted rolling procedure for selecting parameters in online settings, which offers a more systematic (though substantially more complex) approach to tuning.

\subsection{Comparisons with Prior Estimators}\label{subsec:univariate}
In this section, we present a comparative analysis of our F-SGD estimator against 
several existing estimators. Section~\ref{subsubsec:backfitting} contrasts F-SGD 
with online smooth backfitting. There, we include the dependence on $p$ 
in the memory and computational complexity, since F-SGD exhibits significant 
advantages when $p$ is large. Section~\ref{subsubsec:sieve} and 
Section~\ref{subsubsec:rkhs} compare F-SGD with Sieve-SGD and kernel SGD, 
respectively. Following the convention in prior work 
\cite{zhang2022sieve, 4406a3ce-1e59-3ed8-8c25-9f7d75b81f14}, we include only the 
dependence on $n$ in these cases (indeed, all methods share the same dependence on $p$). 
The advantages of F-SGD in these comparisons primarily arise from its more favorable 
exponent with respect to the sample size $n$. For simplicity, we measure memory complexity in terms of the number of 
floating-point numbers that need to be stored.\footnote{In practice, computers 
cannot represent arbitrary real numbers with infinite precision, and thus 
round-off errors may occur. As shown in Section F of \cite{zhang2022sieve}, 
when accounting for the number of bits required to preserve the optimal 
convergence rate under round-off errors, Sieve-SGD requires $O(\log n)$ times 
more memory resources. By the same reasoning, F-SGD also incurs an additional 
$O(\log n)$ factor in space complexity. We refer interested readers to their paper for further details.}

\subsubsection{Comparisons with Online Smooth Backfitting}\label{subsubsec:backfitting}

\cite{yang2023online} proposed an online smooth backfitting method for nonparametric additive models. The key idea is to approximate the nonlinear estimating equations by a suitable order expansion and to store the corresponding coefficients as sufficient statistics, which are updated online via a dynamic candidate bandwidth scheme for each block. As discussed in their paper, online smooth backfitting provides significant gains in memory and computational efficiency over batch backfitting 
methods. While powerful, this procedure is rather involved. In contrast, our SGD-based methods are substantially simpler and more straightforward to implement.  

\cite{yang2023online} assume that each component function is twice continuously differentiable, and they show that online smooth backfitting achieves the optimal convergence rate with respect to the sample size $n$. By comparison, F-SGD accommodates a broader range of smoothness conditions and, more importantly, achieves minimax optimality not only in terms of the sample size $n$ but also with respect to the dimension $p$.  

We now compare the memory and computational complexity of the two methods. 
Online smooth backfitting requires storing $L$ sets of sufficient statistics, 
resulting in a space complexity of order $\Theta(Lp^3)$. 
Here $L$ denotes the number of candidate bandwidth sequences, 
which typically ranges from 5 to 20. 
In contrast, the space complexity of F-SGD is 
$\Theta\bigl(p\, i^{1/(2s+1)}\bigr)$ at the $i$-th iteration. 
When each component function is twice continuously differentiable ($s=2$), 
the complexity of F-SGD is smaller under the mild condition 
$i \lesssim L^5 p^{10}$.

As for computational complexity, in each block online smooth backfitting 
requires matrix inversion and a dual iteration process, 
leading to a cost of order $\Theta(Lp^3 T^{\text{inner}}_k T^{\text{outer}}_i)$, 
where $T^{\text{inner}}_k$ and $T^{\text{outer}}_i$ denote the numbers of inner 
and outer iterations for the $k$-th block. 
In the single-pass setting---i.e., each block contains one sample---and assuming 
$T^{\text{inner}}_k \asymp T^{\text{inner}}$ and 
$T^{\text{outer}}_i \asymp T^{\text{outer}}$ across all blocks, 
the cumulative complexity is 
$\Theta(L n p^3 T^{\text{inner}} T^{\text{outer}})$.\footnote{If each block 
contains a moderate number of samples, the complexity reduces to approximately 
$\Theta(n + LKp^3 T^{\text{inner}} T^{\text{outer}})$, where $K$ is the number 
of blocks. For a fair comparison, we focus on the single-pass case.} 
In contrast, F-SGD incurs computational cost 
$\Theta(p n^{1+1/(2s+1)})$. 
For $s=2$, this is smaller under the mild condition 
$n \lesssim L^5 p^{10}(T^{\text{inner}}T^{\text{outer}})^5$.

\subsubsection{Comparisons with Sieve-SGD}\label{subsubsec:sieve}

While \cite{zhang2022sieve} proposed a version of Sieve-SGD for additive models, as delineated in the update rules \eqref{zhang1} and \eqref{zhang2}, they did so without specifying the truncation levels $J_{ik}$ and without fully investigating how the dimensionality affects its theoretical performance. Their primary emphasis was on instances where the function $f$ belongs to a (multivariate) Sobolev ellipsoid. In a similar vein, we can develop a corresponding version of F-SGD tailored for these situations, which allows for a direct comparison with Sieve-SGD.

To this end, we consider a multivariate orthonormal basis $\{\psi_{j}\}$, where $\psi_j:\RR^p \to \RR$, and define a multivariate Sobolev ellipsoid as:
\begin{align}\label{tilde_W}
\tilde{\mathcal{W}}_p(s,Q, \{\psi_j\}) \coloneqq \bigg\{f:\RR^p \to \RR  : f = \sum_{j=1}^\infty \theta_{f,j}\psi_j, \text{ where } \sum_{j=1}^\infty (j^s\theta_{f,j})^2 < Q^2\bigg\}.
\end{align}
The original Sieve-SGD formulation by \cite{zhang2022sieve} in a non-additive setting is
\begin{align}
   &  \tilde{f}_i(\bX) = \frac{i}{i+1}\tilde{f}_{i-1}(\bX) + \frac{1}{i+1} \tilde g_i(\bX) \label{zhang3}\\
    & \tilde g_i(\bX) = \tilde g_{i-1}(\bX) + \gamma_i\left(Y_i - \tilde g_{i-1}\left(\bX_i\right)\right)\left(\sum_{j=1}^{J_{i}}t_{j}\psi_{j}\left(\bX_i\right)\psi_{j}\left(\bX\right)\right). \label{zhang4}
\end{align}
They showed that when $f$ belongs to $\tilde{\mathcal{W}}_p(s,Q, \{\psi_j\})$,  Sieve-SGD \eqref{zhang3} and \eqref{zhang4}, initialized with $\tilde g_0=\tilde f_0 =0$, converges at the minimax rate $\EE \big[ \norm{\tilde{f}_n - f}^2 \big] = O\big(n^{-2s/(2s+1)}\big)$.

In the same spirit, for $f \in \tilde{\mathcal{W}}_p(s,Q, \{\psi_j\})$, we can adapt F-SGD \eqref{update:l2loss} for the multivariate Sobolev ellipsoid as follows:
\begin{align}
    \hat f_i(\bX) = \hat f_{i-1}(\bX) + \gamma_i\left(Y_i - \hat f_{i-1}\left(\bX_i\right)\right)\left(\sum_{j=1}^{J_{i}}\psi_{j}\left(\bX_i\right)\psi_{j}\left(\bX\right)\right). \label{update:F-SGD}
\end{align}
Importantly, with the same analysis and parameter choices as those in Theorem \ref{thm:fixed}, we are also able to show F-SGD \eqref{update:F-SGD}, initialized with $\hat f_0=0$, converges at the minimax rate $\EE \big[ \norm{\hat{f}_n - f}^2 \big] = O\big(n^{-2s/(2s+1)}\big)$.

Our next theorem considers the case where the input density is not bounded away from zero. As noted in \cite{zhang2022sieve}, the lower bound requirement on the input density (Assumption \ref{asmp:2}) may be an artifact of their proof technique, and so they did not provide theoretical guarantees without this assumption. In contrast, we show that F-SGD with a different parameter selection achieves polynomial rates of convergence even when the input density is not bounded away from zero.

\begin{theorem}\label{thm:polynomial}
    Suppose Assumptions \ref{asmp:1}, \ref{asmp:2'}, \ref{asmp:3} and \ref{asmp:4} hold. Furthermore, suppose $f \in \tilde{\mathcal{W}}_p(s,Q, \{\psi_j\})$ with $s > 1/2$. Set $\gamma_i = Ai^{-(4s+1)/(6s+1)}$, where $A$ is a sufficiently small constant, and $J_i = \lceil i^{\frac{1}{6s+1}} \rceil$. Then the MSE of F-SGD \eqref{update:F-SGD}, initialized with $\hat f_0 = 0$, satisfies
 \begin{align*}
        \EE \left[ \norm{\hat{f}_n - f}^2 \right] = O\left(n^{-\frac{2s-1}{8(6s+1)}}\right).
\end{align*}
\end{theorem}

The proof of this theorem follows the same overall strategy as Theorem~\ref{theorem:l2loss}, 
but with one key difference. When the input density is bounded away from zero, the 
unweighted error
\[
\mathbb{E}\left[\int \big(\hat f_{i-1}(\bx)-f(\bx)\big)^2 d\bx\right]
\]
can be directly controlled by the weighted error
\[
\mathbb{E}\left[\int \big(\hat f_{i-1}(\bx)-f(\bx)\big)^2 p_{\bX}(\bx) d\bx\right] 
= \mathbb{E}\left[\|\hat f_{i-1}-f\|^2\right].
\]
Without this assumption, however, the unweighted error must be tracked separately. 
As a result, the proof requires three recursive inequalities: one for the weighted error, 
one for the unweighted error, and one for the cross term that appears when expanding the 
squared error and couples the estimation error with the one-step update. 
These three recursions are interdependent, and solving them together yields the desired result. 
Full details are given in Section \label{appdx:proof}.

As a closing to this section, we discuss the various advantages that F-SGD offers over Sieve-SGD, despite their similar performance guarantees.

\begin{itemize}
\item Unlike Sieve-SGD, we do not use a component-specific learning rate $t_j$ in \eqref{zhang4} or Polyak averaging step \eqref{zhang3}. This streamlining makes our estimator more intuitive and better motivated by SGD; it is precisely the functional counterpart of SGD applied to the coefficients in a truncated basis expansion of the component functions. 

\item The selected number of dimensions $J_i$, as will be discussed later, impacts the computational and memory storage complexity. In F-SGD \eqref{update:F-SGD}, we choose $J_i = \Theta\big(i^{1/(2s+1)}\big)$. However, in Sieve-SGD \eqref{zhang4}, \cite{zhang2022sieve} opt for $ J_{i} = \Theta \big( i^{1/(2s+1)}\log^2(i)\big) $, which is larger by a logarithmic factor. Regarding the additive model, the specification of the $J_{ik}$ values, as in \eqref{zhang2}, was not explicitly provided. In contrast, as demonstrated in Theorem \ref{theorem:l2loss}, we can choose a homogeneous $J_i$ of order $\Theta\big(i^{1/(2s+1)}\big)$; see \eqref{update:l2loss}.

\item For Sieve-SGD, the component-specific learning rate $t_j = j^{-2\omega}$, where $ 1/2 < \omega < s $, requires knowing a lower bound on $ s $ and comes at the expense of losing a logarithmic factor in the convergence rate. More precisely, the convergence rate is $O\big(n^{-2s/(2s+1)}\log^2 n\big)$ \cite[Theorem 6.3]{zhang2022sieve}.\footnote{It should be noted, however, than if one knows $s$ a priori and uses the ``correct'' component-specific learning rate $t_j = j^{-2s}$, then the minimax optimal rate $O\big(n^{-2s/(2s+1)}\big)$ is attainable.} In contrast, our estimator achieves the minimax optimal rate of $O\big(n^{-2s/(2s+1)}\big)$ in Corollary \ref{cor:W}.

\item In Theorem \ref{thm:fixed}, we use a global learning rate $ \gamma_i = A/i $, where $ A $ is a constant independent of the smoothness parameter $ s $.  In contrast, \cite{zhang2022sieve} opt for $ \gamma_i = \Theta\big(i^{-1/(2s+1)}\big) $ in Sieve-SGD \eqref{zhang4}. Our choice ensures that the dependency on $ s $ is exclusively through $ J_i $, paving the way for the development of an online version of Lepski's method, as elaborated in Section \ref{subsec:Lepski}.
\end{itemize}

\subsubsection{Comparisons with Other Reproducing Kernel Methods}\label{subsubsec:rkhs}
Note that the Sobolev ellipsoid $\tilde{\mathcal{W}}_p(s,Q, \{\psi_j\})$ can be equivalently characterized by a ball in a RKHS. It is therefore worthwhile to draw comparisons between F-SGD and other RKHS methods.

In the context of kernel SGD \citep{4406a3ce-1e59-3ed8-8c25-9f7d75b81f14,6842642,article}, incorporating a new data point $(\bX_{i+1}, Y_{i+1})$ necessitates the evaluation of $i$ kernel functions at $\bX_{i+1}$, resulting in a computational cost of order $\Theta(i)$, if we assume a constant computational cost of $\Theta(1)$ per kernel evaluation. Consequently, the cumulative time complexity of computing the estimator at the $n$-th step grows as $\Theta(n^2)$. Additionally, one needs to store $n$ covariates $\{\bX_i\}_{i=1}^n$ to compute the estimator, which incurs a space expense of $\Theta(n)$. There have been several attempts to improve the computational complexity of kernel based methods \citep{dai2014scalable,JMLR:v17:14-148,JMLR:v20:16-585}, however, their target function is not 
$f$ but, instead, a penalized population risk minimizer.
  
For F-SGD, we only need to store the coefficients of $J_n = \Theta\big(n^{1/(2s+1)}\big)$ basis functions. Assuming the computational cost of evaluating one basis function at one point is $\Theta(1)$, the computational cost at the $i$-th step would be $\Theta(J_i) = \Theta\big(i^{1/(2s+1)}\big)$, and the total computational cost at the $n$-th step would be $\Theta\big(n^{1+1/(2s+1)}\big)$. Both computational and memory storage complexity would be improved by a polynomial factor compared to reproducing kernel methods and, as mentioned previously, by a logarithmic factor compared to Sieve-SGD. For clarity, we include Table~\ref{table1}, which summarizes this comparison.

Moreover, \cite{zhang2022sieve} showed there is no estimator with $o\big(n^{1/(2s+1)}\big)$-sized space expense that can achieve the minimax-rate for estimating $f \in \tilde{\mathcal{W}}_p(s,Q, \{\psi_j\})$ \cite[Theorem 6.5]{zhang2022sieve}. Thus, F-SGD, with its $\Theta\big(n^{1/(2s+1)}\big)$-sized space expense, is (nearly) space-optimal.

\begin{table}[ht]
\centering
\renewcommand{\arraystretch}{1.2}
\begin{tabular}{p{4.2cm}p{4.2cm}p{4.2cm}}
\toprule
Method & Storage Complexity & Computational Complexity \\
\midrule
F-SGD & $\Theta(n^{1/(2s+1)})$ & $\Theta(n^{1+1/(2s+1)})$ \\
Sieve-SGD \citep{zhang2022sieve} & $\Theta(n^{1/(2s+1)}\log^2 n)$ & $\Theta(n^{1+1/(2s+1)}\log^2 n)$ \\
Kernel SGD \citep{4406a3ce-1e59-3ed8-8c25-9f7d75b81f14,6842642} & $\Theta(n)$ & $\Theta(n^2)$  \\
\bottomrule
\end{tabular}
\caption{\small Comparison of storage and cumulative computational complexity under the Sobolev ellipsoid $ \tilde{\mathcal{W}}_p(s, Q, \{\psi_j\}) $. The storage complexity refers to a single iteration, while the computational complexity accounts for the cumulative cost over all iterations.} \label{table1}
\end{table}

\section{Numerical Experiments}\label{sec:numerical}

\subsection{F-SGD for Additive Models on Simulated Data}
In this section, we present the numerical performance of F-SGD, utilizing the parameter selections outlined in Theorem \ref{thm:fixed} and Theorem \ref{theorem:l2loss}.

We first apply Theorem \ref{thm:fixed} and  consider the regression function
\begin{align}\label{sim}
f(\bX) = 5 + \sum_{k=1}^p\big[ \big(X^{(k)}\big)^4 - 2\big(X^{(k)}\big)^3 + \big(X^{(k)}\big)^2 - 1/30\big].
\end{align}
where each component function is the fourth Bernoulli polynomial. We employ the trigonometric basis, excluding the constant basis function, for each component function.  In this case, the smoothness parameter is $s = 2$. Consequently, $f$ belongs to $\mathcal{W}_p(2, Q)$ for some constant $Q$. We consider two different data generating processes: (a) $\bX$ has an uniform distribution over $[0,1]^p$, and (b) $X^{(k)} = \big(U^{(k-1)}+U^{(k)}\big)/2$ for $k = 1,2,\dots,p$ where $U^{(0)}=U^{(p)}$ and $\boldsymbol{U} = \big(U^{(1)},U^{(2)},\dots,U^{(p)} \big)^T$ is uniform on $[0,1]^p$. We set the noise $\varepsilon$ to follow a uniform distribution over $[-0.02, 0.02]$.

Figure \ref{Fig:fixed_p} provides empirical evidence that F-SGD converges at the minimax optimal rate, when the parameters are chosen according to Theorem \ref{thm:fixed}. When $\bX$ follows data generating process (a), we set $A=1$ and $B=0.5$ across three cases corresponding to $p=5 $, $ 30, $ and $80$. When $\bX$ follows data generating process (b), we set $A = 3$ and $ B=0.4$ for $p=5$; $A=7$ and $B=0.4$ for $p=30$; and $A=15 $ and $ B=0.3$ for $p=80$. 
In Figure \ref{Fig:fixed_p}(b), we observe that the performance of F-SGD exhibits an initial plateau phase, wherein the MSE remains relatively constant and does not exhibit significant reduction. This behavior is due to the fact that $J_i = \lfloor Bi^{1/(2s+1)}\rfloor$ is zero for small values of $i$, resulting in only the intercept term being estimated. With larger values of $p$, this plateau persists for more iterations, as a smaller $B$ is required when $p$ increases, thus extending the duration before more complex model components are introduced. However, as $i$ grows beyond this initial stage, the MSE begins to show a trend of minimax optimal convergence.
\begin{figure}[htbp]
    \centering
    \subfigure[$\bX$ follows data generating process (a) ]{\includegraphics[width=0.48\textwidth]{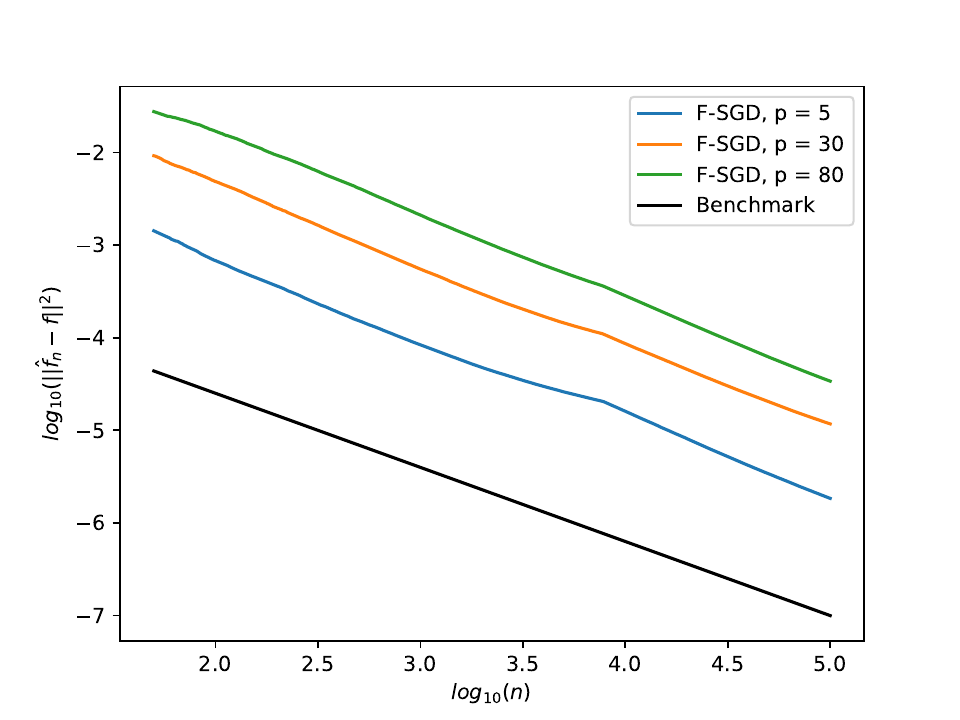}} 
    \subfigure[$\bX$ follows data generating process (b)]{\includegraphics[width=0.48\textwidth]{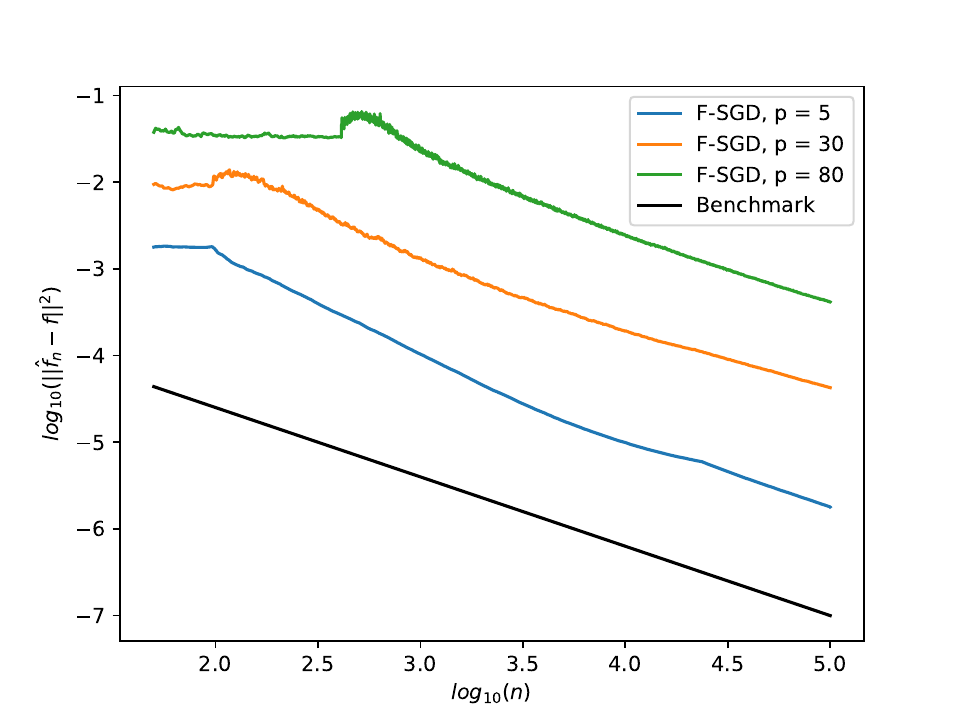}}
    \caption{\small $\log_{10} \norm{\hat{f}_n - f}^2$ against $\log_{10} n$.   The benchmark (black line) has slope $-2s/(2s+1)= -4/5$, which represents the minimax optimal rate. Each curve is calculated as the average of 20 repetitions. (a) $\bX$ follows data generating process (a). (b) $\bX$ follows data generating process (b).}\label{Fig:fixed_p}
\end{figure}

To examine how the MSE depends on $p$, we select parameters according to Theorem \ref{theorem:l2loss} and continue to use the regression function \eqref{sim}. The covariates $ \bX $ are generated from a uniform distribution over $ [0,1]^p $, and the noise $ \varepsilon $ is drawn from a uniform distribution over $ [-0.02, 0.02] $.  We set $A_2 = 5A_1 $, $ A_1=1$, and $B=0.5$ across all three cases corresponding to $p=5$, $ 30 $, and $80$.

Figure \ref{Fig:growing_p}(a) shows the convergence rate in terms of the sample size $ n $. The convergence exhibits three distinct behaviors, each corresponding to the three stages of training.  In the first stage, the MSE remains unchanged due to the absence of updates to $ \hat f_i $. This is followed by the second stage, a transitional period during which moderate learning occurs. For clarity, vertical dotted lines are used to indicate the second stage for $p = 80$, which spans from $n = 161$ to $n = 239$. 
Ultimately, in the third stage, the MSE attains minimax optimal convergence with respect to $ n $. Figure \ref{Fig:growing_p}(b) provides the convergence rate in terms of the dimension $p$. We select $p$ from the range $\{29,32,35,\dots,71\}$ and evaluate the MSE at $n = 10^5$. It is seen that the MSE convergence is close to minimax optimal in terms of $p$.
\begin{figure}[htbp]
    \centering
    \subfigure[Convergence rate in terms of $n$]{\includegraphics[width=0.48\textwidth]{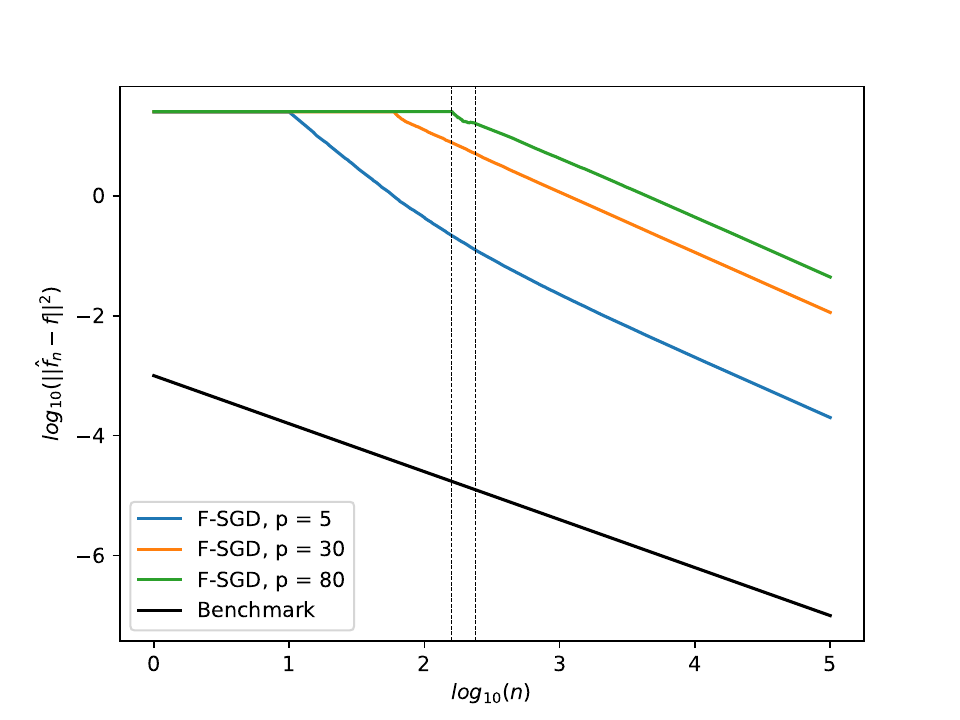}} 
    \subfigure[Convergence rate in terms of $p$]{\includegraphics[width=0.48\textwidth]{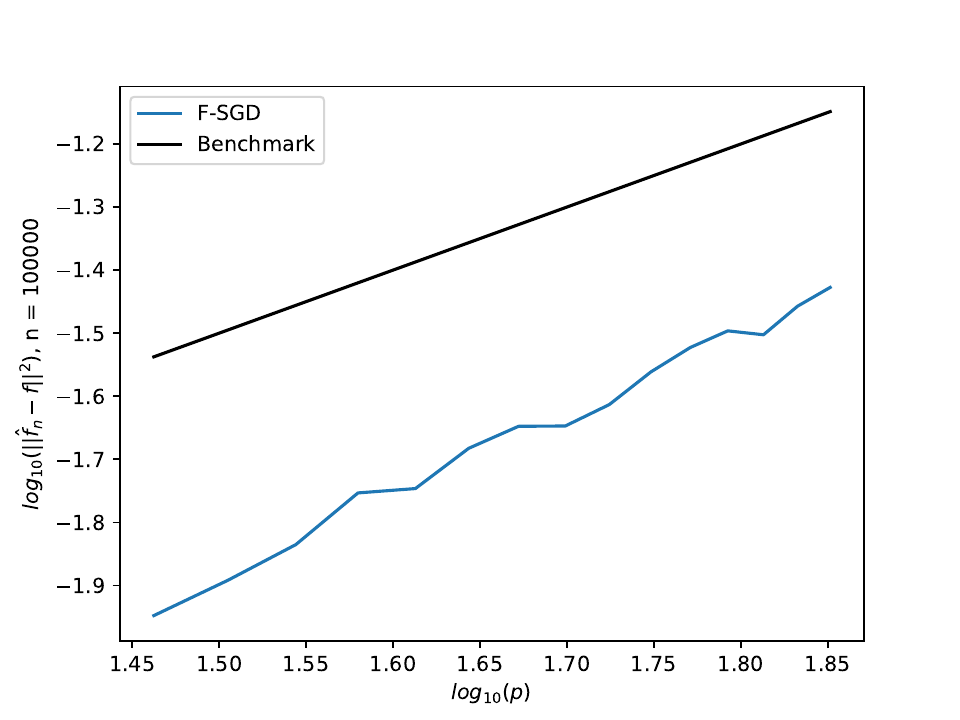}}
    \caption{\small (a) $\log_{10} \norm{\hat{f}_n - f}^2$ against $\log_{10} n$.   The benchmark (black line) has slope $-2s/(2s+1)= -4/5$, which represents the minimax optimal rate in terms of $n$. Vertical dotted lines are used to indicate the second stage for $p = 80$, which spans from $n = 161$ to $n = 239$.
    (b) $\log_{10} \norm{\hat{f}_n - f}^2$ against $\log_{10} p$ when $n=10^5$.   The benchmark (black line) has slope 1, which represents the minimax optimal rate in terms of $p$. Each curve is calculated as the average of 20 repetitions.}\label{Fig:growing_p}
\end{figure}

\subsection{F-SGD vs. Sieve-SGD for Univariate Models on Simulated Data}\label{subsec:sim_compare}
In this section, we evaluate the performance F-SGD in comparison to Sieve-SGD \citep{zhang2022sieve}  when $ p = 1$. The parameters provided by \cite{zhang2022sieve} are specifically tailored for the case when $f \in \tilde W_p(s,Q)$, as defined in \eqref{tilde_W}. Therefore, we likewise assume $f$ belongs to $ \tilde W_p(s,Q)$.   It is worth noting that \cite{zhang2022sieve} included two other methods: kernel ridge regression (KRR) \citep{wainwright_2019} and kernel SGD \citep{4406a3ce-1e59-3ed8-8c25-9f7d75b81f14}. The results indicated that Sieve-SGD exhibited comparable statistical performance but outperformed the other methods in terms of computational efficiency. In this section, we focus on comparing F-SGD and Sieve-SGD, but we do not include KRR and kernel SGD in our analysis. As elaborated in Section \ref{subsubsec:rkhs}, F-SGD exhibits a time complexity of $\Theta\big(n^{1+1/(2s+1)}\big)$, leading us to anticipate that it will also provide enhanced computational efficiency.

We consider the regression function $f(x) = B_4(x) =
x^4 - 2x^3 + x^2 - 1/30$, which is the fourth Bernoulli polynomial, and we utilize the trigonometric basis.  The smoothness parameter is $s = 2$.  Hence, $f$ belongs to Sobolev ellipsoid $W(2, Q)$ for some constant $Q$. We consider two different scenarios regarding the distribution of $X$: (a) $X$ has an uniform distribution over $[0,1]$, and (b) $X$ has an uniform distribution over $[0.25,0.75]$. It is worth mentioning that while the trigonometric basis functions are orthonormal with respect to the Lebesgue measure in scenario (a), they are not orthonormal in scenario (b). Additionally, we apply two different noise levels in these scenarios: In scenario (a), we use small noise, where $\varepsilon$ follows a uniform distribution over $[-0.02, 0.02]$. In scenario (b), we use  large noise, where $\varepsilon$ follows a uniform distribution over $[-0.2, 0.2]$. This setup mirrors that of Example 1 in the simulation study in \cite{zhang2022sieve}.

In the case of F-SGD, for scenario (a), we set $\gamma_i = 3/i$ and $J_i = \lfloor i^{1/(2s+1)} \rfloor = \lfloor i^{0.2} \rfloor$. In scenario (b), we set $\gamma_i = 3/i$ and $J_i = \lfloor 0.8i^{1/(2s+1)} \rfloor  = \lfloor 0.8i^{0.2} \rfloor$. As for Sieve-SGD, we adopt the same parameters as presented in \cite{zhang2022sieve}. Specifically, we use $\gamma_i = 3i^{-0.2}$ and set $J_n = n^{0.21}$. Additionally, we choose $t_i = j^{-2s} =j^{-4}$, which corresponds to the oracle component-specific learning rate (Theorem 6.1 in \cite{zhang2022sieve}).

Figure \ref{Fig:uniform} provides empirical evidence that F-SGD converges at the minimax optimal rate and exhibits performance closely resembling Sieve-SGD when equipped with an appropriately chosen component-specific learning rate.

\begin{figure}[htbp]
    \centering
    \subfigure[Uniform distribution over 0 to 1 ]{\includegraphics[width=0.48\textwidth]{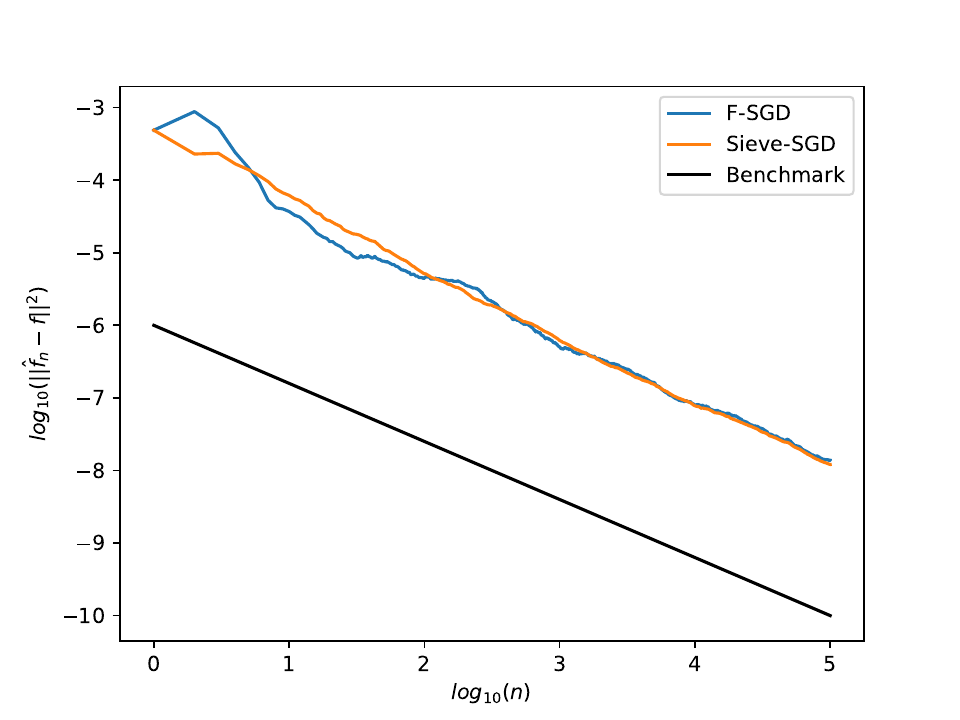}} 
    \subfigure[Uniform distribution over 0.25 to 0.75]{\includegraphics[width=0.48\textwidth]{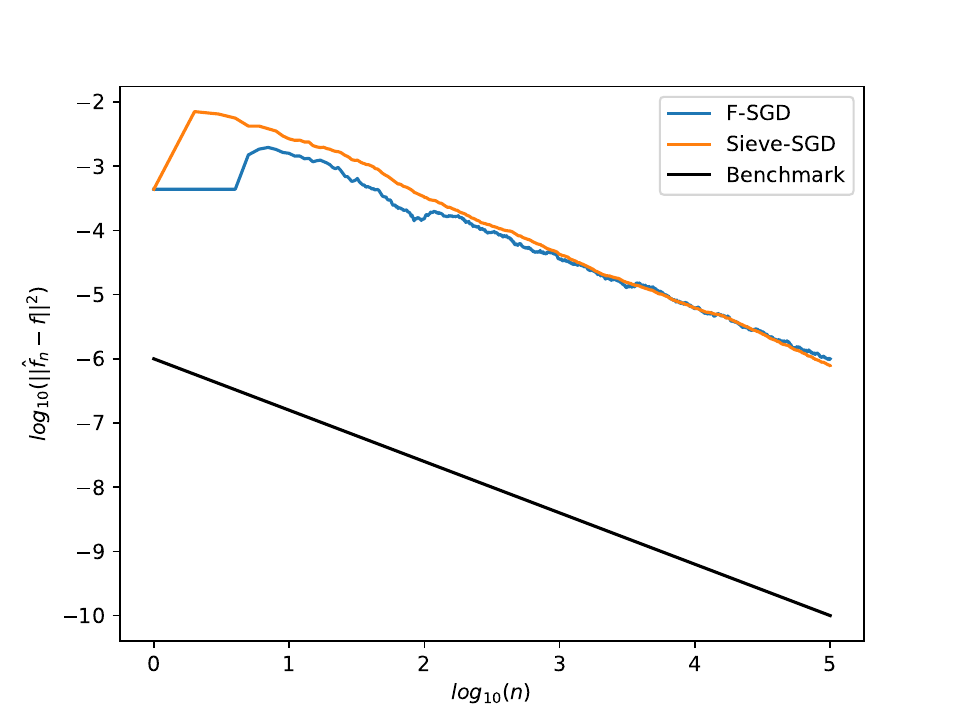}}
    \caption{\small $\log_{10} \norm{\hat{f}_n - f}^2$ against $\log_{10} n$.   The benchmark (black line) has slope $-2s/(2s+1)= -4/5$, which represents the minimax optimal rate. Each curve is calculated as the average of 100 repetitions. (a) $X$ is uniformly distributed over [0, 1]. (b) $X$ is uniformly distributed over [0.25, 0.75].}\label{Fig:uniform}
\end{figure}

\section{Discussion and Conclusion}\label{sec:discuss}
This paper introduced a novel estimator based on SGD for learning  nonparametric additive models.  We demonstrated its optimal theoretical performance, favorable computational complexity, and minimal memory storage requirements under some simple conditions on the data generating process. In this section, we briefly discuss some potential future extensions.  
\subsection{F-SGD with Lepski's Method}\label{subsec:Lepski}
One important advantage of F-SGD lies in the fact that the global learning rate $ \gamma_i $ can be chosen to be independent of $ s $, namely, $\gamma_i = \Theta(1/i)$, if we do not require the optimal convergence for $p$ (see Theorem \ref{thm:fixed}). 
Because of this selection, the sole dependence on $ s $ is through $ J_i $, which enables us to develop an online version of Lepski's method.

Lepski's method is widely used in adaptive nonparametric estimation when the smoothness parameter is unknown \citep{doi:10.1137/1135065}. In the classical batch setting, it achieves a rate slower than the nonadaptive minimax rate by only a logarithmic factor.  At each iteration step, let $s_{i,a}, s_{i,b} \in \mathcal{S}_i=\{s_0, s_0 + 1/\log i, s_0 + 2/\log i, \dots, s_1\}$ and define $J_{i,a} = \lfloor Bi^{\frac{1}{2s_{i,a}+1}} \rfloor$ and $ J_{i,b} = \lfloor Bi^{\frac{1}{2s_{i,b}+1}} \rfloor$ for some constant $B$. We then introduce $\hat f_i^{J_{i,l}}$, where $l \in \{a,b\}$,  as: 
\begin{align}\label{f_i:lepski}
    \hat f_i^{J_{i,l}}(\bX) = \hat f_{i-1}(\bX) + \gamma_i\left(Y_i - \hat f_{i-1}\left(\bX_i\right)\right)\left(1+\sum_{k=1}^p\sum_{j=1}^{J_{i,l}}\psi_{jk}\left(X_i^{(k)}\right)\psi_{jk}\left(\bX\right)\right).
\end{align}
To adaptively select the smoothing parameter, we adopt the core idea of Lepski's method at the current update. That is, we select the largest $s_{i,a} \in \mathcal{S}_i$ that satisfies the following inequality for all $s_{i,b} < s_{i,a}$, where $s_{i,b} \in \mathcal{S}_i$:
\begin{align}\label{update:lepski}
\left\|\hat f_i^{J_{i,b}} - \hat f_i^{J_{i,a}}\right\|^2 \leq \left(\frac{i}{\log i}\right)^{-\frac{2s_{i,b}}{2s_{i,b}+1}}.
\end{align}
Since $p_{\bX}(\bx) \leq C_2$, as per Assumption \ref{asmp:2}, we have
\begin{align*}
    \left\|\hat f_i^{J_{i,b}} - \hat f_i^{J_{i,a}}\right\|^2 \leq C_2\gamma_i^2 \left(Y_i - \hat{f}_{i-1}\left(X_i\right)\right)^2 \sum_{k=1}^p\sum_{j = J_{i,a}+1}^{J_{i,b}}\psi_{jk}^2\left(X_i^{(k)}\right).
\end{align*}
If we set $\gamma_i = A/i$ and absorb $C_2$ into $A$, we can simplify \eqref{update:lepski} and select the smoothing parameter by identifying the largest $s_{i,a} \in \mathcal{S}_i$ such that, for all $s_{i,b} < s_{i,a}$ where $s_{i,b} \in \mathcal{S}_i$, the following condition holds:
\begin{align}\label{ineq:cond_lepski}
    \gamma_i^2 \left(Y_i - \hat{f}_{i-1}\left(X_i\right)\right)^2 \sum_{k=1}^p\sum_{j = J_{i,a}+1}^{J_{i,b}}\psi_{jk}^2\left(X_i^{(k)}\right) \leq \left(\frac{i}{\log i}\right)^{-\frac{2s_{i,b}}{2s_{i,b}+1}}.
\end{align}
We summarize the procedure in Algorithm \ref{alg:1}.
\begin{algorithm}[ht]
\caption{F-SGD with Lepski's method \label{alg:1}}
For $i=1$ to $n$:

\quad Set $\gamma_i = A/i$ and $J_{i,l} = \lfloor Bi^{\frac{1}{2s_{i,l}+1}} \rfloor$ where $l \in \{a,b\}$ for some appropriate constants $A$ and $B$.

\quad Select the largest $s_{i,a} \in \mathcal{S}_i$ such that, for all $s_{i,b} < s_{i,a}$ where $s_{i,b} \in \mathcal{S}_i$, the condition \eqref{ineq:cond_lepski} holds.

\quad Update $\hat f_i$ according to \eqref{f_i:lepski} with $l=a$.
\end{algorithm}

The theoretical investigation of F-SGD with Lepski's method is beyond the scope of this paper, though we undertake a numerical study in the next section to highlight its potential. The results indicate that F-SGD with Lepski's method exhibits favorable performance.

\subsubsection{Simulations for F-SGD with Lepski's Method}

In this section, we present some numerical results for F-SGD with Lepski's method. We again consider the function $f(x) = B_4(x) =
x^4 - 2x^3 + x^2 - 1/30$ and use the trigonometric basis. In contrast to previous sections, here we assume the smoothness parameter $s$ is unknown, but that it lies within the range $[s_0, s_1]$ where $s_0=0.5$ and $s_1=8$. Similar to the simulation example in Section \ref{subsec:sim_compare}, the covariate $X$ is generated from two different distributions: (a) $X$ has a uniform distribution over $[0,1]$, and (b) $X$ has a uniform distribution over $[0.25,0.75]$. Also, we apply two different noise levels in these scenarios: In scenario (a), we use small noise, where $\varepsilon$ follows a uniform distribution over $[-0.02, 0.02]$. In scenario (b), we use  large noise, where $\varepsilon$ follows a uniform distribution over $[-0.2, 0.2]$.

We use F-SGD with Lepski's method (Algorithm \ref{alg:1}) to update $\hat f_i$. For scenario (a), we set $A=B=3$, and for scenario (b), we set $A = 4 $ and $ B=2$. Figure \ref{Fig:Lepski} illustrates that even without knowing $s$ a priori, F-SGD with Lepski's method can still achieve near minimax optimal convergence.

\begin{figure}[htbp]
    \centering
    \subfigure[Uniform distribution over 0 to 1]{\includegraphics[width=0.48\textwidth]{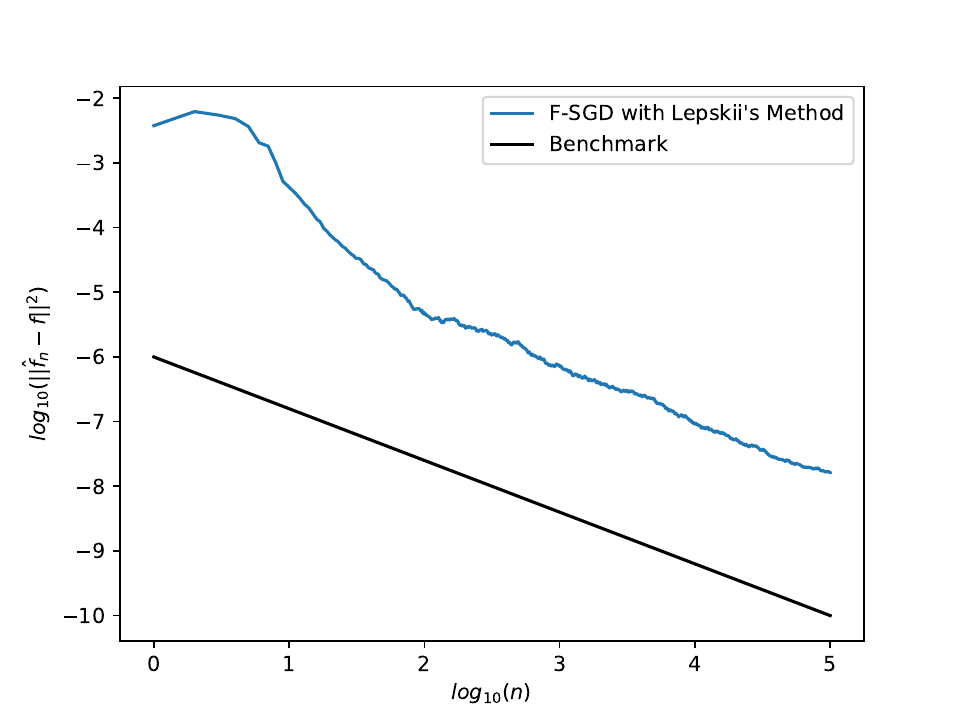}} 
    \subfigure[Uniform distribution over 0.25 to 0.75]{\includegraphics[width=0.48\textwidth]{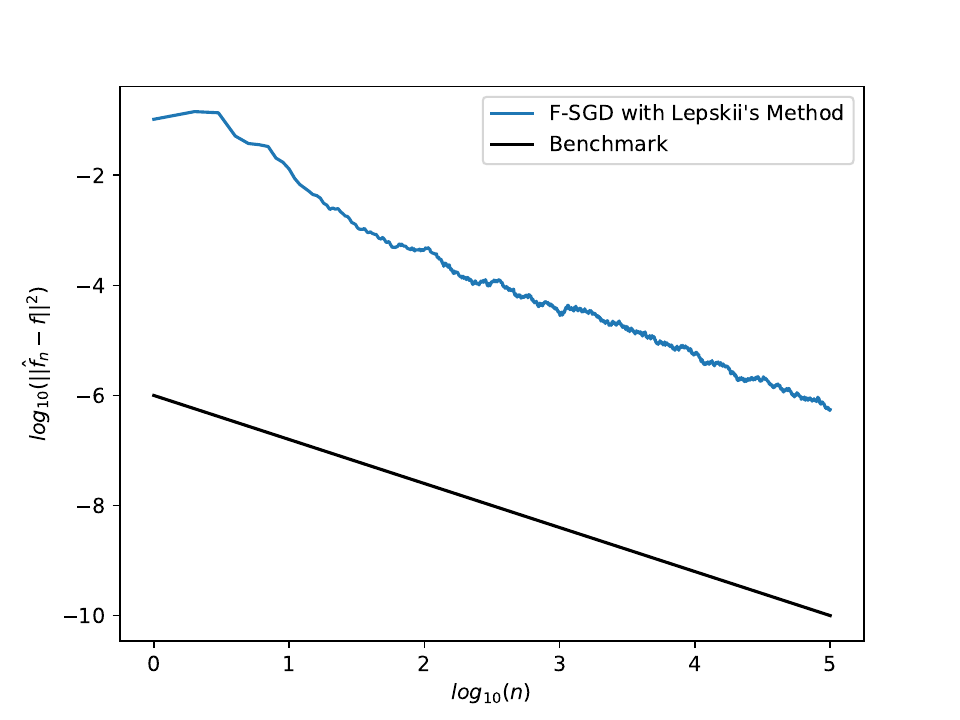}}
    \caption{\small $\log_{10} \norm{\hat{f}_n - f}^2$ against $\log_{10} n$, where $\hat f_n$ is the F-SGD with Lepski's method. The benchmark (orange line) has slope $-2s/(2s+1)= -4/5$, which represents the optimal rate. Each curve is calculated as the average of 30 repetitions. (a) $X$ is uniformly distributed over [0, 1]. (b) $X$ is uniformly distributed over [0.25, 0.75].}\label{Fig:Lepski}
\end{figure}

\subsection{General Convex Loss Functions}
Suppose we aim to minimize 
\begin{align*}
    \min_{f \in \mathcal{F}_p} \EE[\ell(Y, f(\bX))],
\end{align*}
where the loss function $\ell(u,v)$ is convex with respect to $v$ for any $u$. Following the same principle as F-SGD \eqref{update:l2loss}, we  can update $\hat f_i$ as follows:
\begin{align*}
    &\hat{f}_i(\bX) = \hat{f}_{i-1}(\bX) - \gamma_i\frac{\partial \ell(u,v)}{\partial v}\vert_{u = Y_i,\; v = \hat f_{i-1}(\bX)} \left(1+\sum_{k=1}^p\sum_{j=1}^{J_i}\psi_{jk}\left(X_i^{(k)}\right)\psi_{jk}\left(\bX\right)\right).
\end{align*}
Exploring the selection of parameters and characterizing the theoretical performance of this estimator are interesting directions for future research. Due to its simplicity, extensions based on our approach may be more amenable to theoretical analysis compared to, for example, Sieve-SGD \citep{zhang2022sieve}.

\section{Proofs}\label{appdx:proof}

In this section, we provide proofs of Theorem \ref{theorem:l2loss} and Theorem \ref{thm:fixed}. The proof of Theorem \ref{thm:polynomial} is provided in the Supplementary Material \citep{chen2025supplement}. Before proceeding, we first state some useful lemmas. The first lemma addresses a simple recursive inequality that will be applied repeatedly throughout the proofs.
\begin{lemma}\label{lem}
    Suppose for $i\geq i_0$, we have $e_i \leq a_ie_{i-1} + b_i$ with $a_i,b_i,e_i >0$. Then for any $i\geq i_0$,
    \begin{align}\label{lem:inq}
    e_i \leq e_{i_0-1}\prod_{k=i_0}^ia_k + \sum_{k=i_0}^{i-1}\left(\prod_{j=k+1}^i a_j \right)b_k + b_i.
    \end{align}
\end{lemma}
\begin{proof}
    This lemma can be proved by induction. Note that the inequality \eqref{lem:inq} holds when $i=i_0$. In addition, suppose \eqref{lem:inq} holds for $i$. Then,
\begin{align*}
    e_{i+1} &\leq a_{i+1}e_{i} + b_{i+1}\\
    &\leq a_{i+1}\left( e_{i_0-1}\prod_{k=i_0}^ia_k + \sum_{k=i_0}^{i-1}\left(\prod_{j=k+1}^i a_j \right)b_k + b_i \right) + b_{i+1}\\
    &= e_{i_0-1}\prod_{k=i_0}^{i+1}a_k + \sum_{k=i_0}^{i}\left(\prod_{j=k+1}^{i+1} a_j \right)b_k + b_{i+1}. \tag*{\qedhere}
\end{align*}
\end{proof}
We next introduce several notations before stating the second lemma. Let
\begin{align*}
    e_i &= \mathbb{E}\left[\|\hat{f}_i - f\|^2\right], 
    & v_i &= \mathbb{E}\left[ \int \left( \hat{f}_{i}(\boldsymbol{x}) - f(\boldsymbol{x}) \right)^2 d\boldsymbol{x} \right].
\end{align*}
Define
\begin{align}
    y_{i-1}^2 &= \mathbb{E}\left[ \left| \langle \hat{f}_{i-1} - f, 1 \rangle \right|^2 \right] 
    + \sum_{k=1}^p \sum_{j=1}^{J_i} \mathbb{E}\left[ \left| \langle \hat{f}_{i-1} - f, \psi_{jk} \rangle \right|^2 \right].\label{def:yi}
\end{align}
Throughout the proof, the expectation operator $\EE$ is taken with respect to all of the randomness, unless it is explicitly stated otherwise.

Here, $e_i$ is the quantity we ultimately aim to bound, $v_i$ is the $\ell_2$ error under the uniform measure, and $y_{i-1}$ is an intermediate term that arises from the cross term when expanding $\|\hat{f}_i - f\|^2$. The next lemma provides a lower bound for $y_{i-1}^2$.
\begin{lemma}\label{lem:general_bound}
Let $f_k = \sum_{j=1}^\infty \theta_{f,j}^{(k)} \psi_{jk}$ with 
$\theta_{f,j}^{(k)} = \int f_k(x)\psi_{jk}(x)\,dx$, and set 
$f_{k,i} = \sum_{j=1}^{J_i} \theta_{f,j}^{(k)}\psi_{jk}$, 
$\check{f}_i = \alpha + \sum_{k=1}^p f_{k,i}$.
For any $g = \tau + \sum_{k=1}^p g_k$ with $g_k \in \mathcal{H}_1(s)$ 
and $\tau \in \mathbb{R}$, also write 
$g_k = \sum_{j=1}^\infty \theta_{g,j}^{(k)} \psi_{jk}$, 
$g_{k,i} = \sum_{j=1}^{J_i} \theta_{g,j}^{(k)}\psi_{jk}$, 
and $\check{g}_i = \tau + \sum_{k=1}^p g_{k,i}$. Then
\begin{align}
y_{i-1}^2
\geq~\frac{e_{i-1}}{v_{i-1}}\Bigg(\frac{e_{i-1}}{2}- 4\left(\|f - g\| + \|\check g_i - \check f_i\|\right)^2 - \frac{C_2\sum_{k=1}^p\|g_k\|_s^2}{J_i^{2s}}\Bigg). \label{ineq:general_bound}
\end{align}
\end{lemma}
\begin{proof}[Proof of Lemma \ref{lem:general_bound}]
Since $ \hat{f}_{i-1} $ and $ \check f_i $ belong to $\text{span}\{\psi_0,\,\psi_{jk}: j=1,2,\dots,J_i \text{ and } k = 1,2,\dots,p\}  $ where $\psi_0 \equiv 1$, we can write $ \hat{f}_{i-1}-\check f_i = \hat\alpha_{i-1}+\sum_{k=1}^p\sum_{j=1}^{J_i} \hat\beta_{i-1,j}^{(k)}\psi_{jk} $, where 
\begin{align}
\hat\beta_{i-1,j}^{(k)} = \int \left(\hat{f}_{i-1}(\boldsymbol{x})-\check{f}_i(\boldsymbol{x})\right)\psi_{jk}\left(x^{(k)}\right) d\boldsymbol{x}, \qquad \hat\alpha_{i-1} = \int \left(\hat{f}_{i-1}(\boldsymbol{x})-\check{f}_i(\boldsymbol{x})\right) d\boldsymbol{x}.    \label{notation:lem} 
\end{align}
Hence, we have
\begin{align*}
    &\left|\EE\left[\innerx{\hat{f}_{i-1}-f}{\hat{f}_{i-1}-\check f_i}\right]\right|^2\notag\\ 
    &=~ \bigg(\EE\left[\int \left(\hat{f}_{i-1}(\bx)-f(\bx)\right)\hat\alpha_{i-1}p_{\bX}(\bx)d\bx \right]\notag\\
    &\qquad + \sum_{k=1}^p\sum_{j=1}^{J_i}\EE\left[\int \left(\hat{f}_{i-1}(\bx)-f(\bx)\right)\hat\beta_{i-1,j}^{(k)}\psi_{jk}\left(x^{(k)}\right)p_{\bX}(\bx)d\bx \right]
 \bigg)^2.
\end{align*}
Using the  Cauchy-Schwarz inequality, we know 
\begin{align}
    &\left|\EE\left[\innerx{\hat{f}_{i-1}-f}{\hat{f}_{i-1}-\check f_i}\right]\right|^2\notag\\ 
 &\leq \Bigg(\EE\left[ \left(\int \left(\hat{f}_{i-1}(\bx)-f(\bx)\right)p_{\bX}(\bx)d\bx\right)^2 \right]\notag\\
 &~ + \sum_{k=1}^p\sum_{j=1}^{J_i}\EE\left[ \left(\int \left(\hat{f}_{i-1}(\bx)-f(\bx)\right)\psi_{jk}\left(x^{(k)}\right)p_{\bX}(\bx)d\bx\right)^2 \right]   \Bigg) \EE\Bigg[\hat\alpha_{i-1}^2+ \sum_{k=1}^p\sum_{j=1}^{J_i} \Big(\hat\beta_{i-1,j}^{(k)}\Big)^2\Bigg] \notag \\
 &=~  \Bigg(\mathbb{E}\left[\left|\innerx{\hat{f}_{i-1}-f}{1}\right|^2\right]+\sum_{k=1}^p\sum_{j=1}^{J_i}\mathbb{E}\left[\left|\innerx{\hat{f}_{i-1}-f}{\psi_{jk}}\right|^2\right]\Bigg)  \EE\Bigg[\hat\alpha_{i-1}^2+ \sum_{k=1}^p\sum_{j=1}^{J_i} \Big(\hat\beta_{i-1,j}^{(k)}\Big)^2\Bigg]\\
 &=~y_{i-1}^2\EE\Bigg[\hat\alpha_{i-1}^2+ \sum_{k=1}^p\sum_{j=1}^{J_i} \Big(\hat\beta_{i-1,j}^{(k)}\Big)^2\Bigg], \label{ineq:inter2}
\end{align}
with
\begin{align}
     &\EE\Bigg[\hat\alpha_{i-1}^2+ \sum_{k=1}^p\sum_{j=1}^{J_i} \Big(\hat\beta_{i-1,j}^{(k)}\Big)^2\Bigg] \notag\\ 
    &\leq~ \EE\Bigg[ \int \left(\hat\alpha_{i-1}+\sum_{k=1}^p\sum_{j=1}^{J_i}\hat\beta_{i-1,j}^{(k)}\psi_{jk}\left(x^{(k)}\right) - \sum_{k=1}^p\sum_{j = J_i+1}^\infty \theta_{f,j}^{(k)}\psi_{jk}\left(x^{(k)}\right)\right)^2 d\bx\Bigg] \notag\\
    &= v_{i-1}. \label{ineq:v}
\end{align}
Furthermore, note that 
\begin{align}
    &\left|\EE\left[\innerx{\hat{f}_{i-1}-f}{\hat{f}_{i-1}-\check f_i}\right]\right|^2\notag\\ 
    &=~ \left|\EE\left[\normx{\hat{f}_{i-1}-f}^2\right] + \EE\left[\innerx{\hat{f}_{i-1}-f}{f-\check f_i}\right]\right|^2  \notag\\
    &\geq~ \left(\EE\left[\normx{\hat{f}_{i-1}-f}^2\right]\right)^2 + 2\left(\EE\left[\normx{\hat{f}_{i-1}-f}^2\right]\right) \left(\EE\left[\innerx{\hat{f}_{i-1}-f}{f-\check f_i}\right]\right). \label{new3}
\end{align}
By the AM-GM inequality, we know 
\begin{align*}
    2\EE\left[\innerx{\hat{f}_{i-1}-f}{f- \check f_i}\right] \geq -\frac{1}{2}\EE\left[\normx{\hat{f}_{i-1}-f}^2\right] - 2\normx{f-\check f_i}^2.
\end{align*}
Hence, continuing from \eqref{new3}, we obtain
\begin{align}
    &\left|\EE\left[\innerx{\hat{f}_{i-1}-f}{\hat{f}_{i-1}-\check f_i}\right]\right|^2\notag\\
    &\geq~  \left(\EE\left[\normx{\hat{f}_{i-1}-f}^2\right]\right)^2 - \left(\EE\left[\normx{\hat{f}_{i-1}-f}^2\right]\right) \left(\frac{1}{2}\EE\left[\normx{\hat{f}_{i-1}-f}^2\right] + 2\normx{f-\check f_i}^2\right) \notag\\
    &=~ \frac{1}{2}\left(\EE\left[\normx{\hat{f}_{i-1}-f}^2\right]\right)^2 - 2\left(\EE\left[\normx{\hat{f}_{i-1}-f}^2\right]\right) \left(\normx{f-\check f_i}^2\right)\\
    &=~ \frac{1}{2}e_{i-1}^2 - 2e_{i-1} \normx{f-\check f_i}^2.
    \label{ineq:inter4}
\end{align}
Note that
\begin{align}
    \normx{f - \check f_i}^2\leq  \left( \normx{f - g} + \normx{g - \check g_i} + \normx{\check g_i - \check f_i} \right)^2, \label{ineq:f-fi}
\end{align}
and
\begin{align}
    \normx{g - \check g_i}^2 &= \int \left( \sum_{k=1}^p\sum_{j=J_i+1}^\infty \theta_{g,j}^{(k)}\psi_{jk}\left(x^{(k)}\right) \right)^2 p_{\bX}(\bx) d\bx \notag\\
    &\leq C_2 \int \left( \sum_{k=1}^p\sum_{j=J_i+1}^\infty \theta_{g,j}^{(k)}\psi_{jk}\left(x^{(k)}\right) \right)^2 d\bx \notag\\
    &= C_2 \sum_{k=1}^p\sum_{j = J_i+1}^\infty \left(\theta_{g,j}^{(k)}\right)^2 \notag\\
    &\leq  C_2\sum_{k=1}^p \frac{1}{J_i^{2s}}\sum_{j = J_i+1}^\infty \left(\theta_{g,j}^{(k)}\right)^2j^{2s} \leq \frac{C_2\sum_{k=1}^p\|g_k\|_s^2}{J_i^{2s}}.\label{ineq:trunc}
\end{align}  
Continuing from \eqref{ineq:f-fi}, we utilize \eqref{ineq:trunc}, and obtain 
\begin{align}
   \normx{f - \check f_i}^2&\leq  2\left(\|f - g\| + \|\check g_i - \check f_i\|\right)^2 + 2\|g - \check g_i \|^2 \notag\\
   &\leq 2\left(\|f - g\| + \|\check g_i - \check f_i\|\right)^2 + \frac{2C_2\sum_{k=1}^p\|g_k\|_s^2}{J_i^{2s}}.\label{ineq:f-fi_last}
\end{align}
Therefore, combining inequalities \eqref{ineq:inter2}, \eqref{ineq:v}, \eqref{ineq:inter4} and \eqref{ineq:f-fi_last}, we have 
\begin{align*}
y_{i-1}^2
\geq~\frac{e_{i-1}^2}{2v_{i-1}}- 4\left(\|f - g\| + \|\check g_i - \check f_i\|\right)^2\frac{e_{i-1}}{v_{i-1}} - \frac{4e_{i-1}C_2\sum_{k=1}^p\|g_k\|_s^2}{v_{i-1}J_i^{2s}}. \tag*{\qedhere}
\end{align*}
\end{proof}

Now we start to prove Theorem \ref{theorem:l2loss}.
\begin{proof}[Proof of Theorem \ref{theorem:l2loss}]
To simplify the notation, at each iteration step $i$, we let $r_i = Y_i - \hat{f}_{i-1}(\boldsymbol{X}_i)$.  Using the recursive relationship \eqref{update:l2loss} for $\hat{f}_i$, we have
\begin{align}
    \normx{\hat{f}_i-f}^2 &= \normx{\hat{f}_{i-1}-f+\gamma_ir_i\left(1+\sum_{k=1}^p\sum_{j=1}^{J_i}\psi_{jk}\left(X_i^{(k)}\right)\psi_{jk}\right)}^2 \notag\\
    &= \normx{\hat{f}_{i-1} - f}^2 + 2\innerx{\hat{f}_{i-1}-f}{\gamma_ir_i\left(1+\sum_{k=1}^p\sum_{j=1}^{J_i}\psi_{jk}\left(X_i^{(k)}\right)\psi_{jk}\right)} \notag\\
    &\qquad +\gamma_i^2r_i^2\normx{1+\sum_{k=1}^p\sum_{j=1}^{J_i}\psi_{jk}\left(X_i^{(k)}\right)\psi_{jk}}^2. \label{equa:fi-f}
\end{align}
Next, we take the expectation of the last term in \eqref{equa:fi-f} and apply the law of iterated expectations:
\begin{align}
    &\EE \left[\gamma_i^2r_i^2\normx{1+\sum_{k=1}^p\sum_{j=1}^{J_i}\psi_{jk}\left(X_i^{(k)}\right)\psi_{jk}}^2\right]\notag\\
     &=~ \EE\sbr{\EE\sbr{\gamma_i^2r_i^2\normx{1+\sum_{k=1}^p\sum_{j=1}^{J_i}\psi_{jk}\left(X_i^{(k)}\right)\psi_{jk}}^2\Bigg|\boldsymbol{X}_i, (\boldsymbol{X}_k, Y_k)_{k=1}^{i-1}}}.\label{new1}
\end{align}
Note that $\EE [r_i^2 \mid \boldsymbol{X}_i, (\boldsymbol{X}_k, Y_k)_{k=1}^{i-1}] \leq \sigma^2 + \big(\hat{f}_{i-1}(\boldsymbol{X}_i) - f(\boldsymbol{X}_i) \big)^2$. Also, we have the inequality
\begin{align}
    \normx{1+\sum_{k=1}^p\sum_{j=1}^{J_i}\psi_{jk}\left(X_i^{(k)}\right)\psi_{jk}}^2 &= \int \left( 1+\sum_{k=1}^p\sum_{j=1}^{J_i}\psi_{jk}\left(X_i^{(k)}\right)\psi_{jk}\left(x^{(k)}\right)\right)^2p_{\bX}(\boldsymbol{x})d\boldsymbol{x}\notag\\
    &\leq C_2\int \left(1+ \sum_{k=1}^p\sum_{j=1}^{J_i}\psi_{jk}\left(X_i^{(k)}\right)\psi_{jk}\left(x^{(k)}\right)\right)^2d\boldsymbol{x}\notag\\
    &=C_2\left(1+\sum_{k=1}^p\sum_{j=1}^{J_i}\psi_{jk}^2\left(X_i^{(k)}\right)\right), \label{new2}
\end{align}
where for the first inequality, we use Assumption \ref{asmp:2} which says that $p_{\bX}(\bx) \leq C_2$, and for the last equality, we use the fact that $\psi_{jk}$ are centered \big($\int \psi_{jk}(x) dx = 0$\big), and for a given $k$, the basis $\{\psi_{jk}\}_j$ is orthonormal. Thus, putting \eqref{new1} and \eqref{new2} together, we have
\begin{align}
    &\EE \left[\gamma_i^2r_i^2\normx{1+\sum_{k=1}^p\sum_{j=1}^{J_i}\psi_{jk}\left(X_i^{(k)}\right)\psi_{jk}}^2\right]\notag\\
    &\leq~ \EE \left[\gamma_i^2(\sigma^2 + (\hat{f}_{i-1}(\boldsymbol{X}_i) - f(\boldsymbol{X}_i) )^2) C_2\left(1+\sum_{k=1}^p\sum_{j=1}^{J_i}\psi_{jk}^2\left(X_i^{(k)}\right)\right)\right] \notag\\
    &\leq~ \gamma_i^2\left(\sigma^2 + \EE \Big[\big\|\hat{f}_{i-1}-f\big\|^2\Big]\right)C_2\left(1+pJ_i M^2\right)\\
    &= \gamma_i^2\left(\sigma^2 + e_{i-1}\right)C_2\left(1+pJ_i M^2\right)
    , \label{ineq:term3}
\end{align}
where the last inequality directly follows from the boundedness of $\psi_{jk}$. 
To compute the second term in \eqref{equa:fi-f}, we first expand the inner product
\begin{align}
    &\EE\left[\innerx{\hat{f}_{i-1}-f}{\gamma_ir_i \left(1+\sum_{k=1}^p\sum_{j=1}^{J_i}\psi_{jk}\left(X_i^{(k)}\right)\psi_{jk}\right)}\right]\notag\\
    &=~\gamma_i \EE \left[\innerx{\hat{f}_{i-1}-f}{\rbr{f(\boldsymbol{X}_i) - \hat{f}_{i-1}(\boldsymbol{X}_i)}\left(1+\sum_{k=1}^p\sum_{j=1}^{J_i}\psi_{jk}\left(X_i^{(k)}\right)\psi_{jk}\right)}\right] \notag \\
&=~-\gamma_i\Bigg(\mathbb{E}\left[\left|\innerx{\hat{f}_{i-1}-f}{1}\right|^2\right]+\sum_{k=1}^p\sum_{j=1}^{J_i}\mathbb{E}\left[\left|\innerx{\hat{f}_{i-1}-f}{\psi_{jk}}\right|^2\right]\Bigg)\\
&=~-\gamma_i y_{i-1}^2, \label{ineq:inter1}
\end{align}
where for the second equality we use the fact that $\hat f_{i-1}$ is independent of $\bX_i$, and for the last equality we recall the notation \eqref{def:yi}.

Then according to \eqref{equa:fi-f}, \eqref{ineq:term3} and \eqref{ineq:inter1}, we know
\begin{align}
    e_i \leq e_{i-1} - 2\gamma_iy_{i-1}^2 + \gamma_i^2\left(\sigma^2 + e_{i-1}\right)C_2\left(1+pJ_i M^2\right). \label{ineq:general_recursion}
\end{align}

To prove Theorem~\ref{theorem:l2loss}, we extend Lemma~\ref{lem:general_bound} 
to obtain a lower bound for $y_{i-1}^2$ independent of $v_{i-1}$.
Recall the notation from Lemma~\ref{lem:general_bound} that  
$f_k = \sum_{j\ge1} \theta_{f,j}^{(k)}\psi_{jk}$, 
$\theta_{f,j}^{(k)} = \int f_k(x)\psi_{jk}(x)\,dx$,  
$f_{k,i} = \sum_{j\le J_i} \theta_{f,j}^{(k)}\psi_{jk}$, 
and $\check{f}_i = \alpha + \sum_{k=1}^p f_{k,i}$.  
Also recall $\hat{\beta}_{i-1,j}^{(k)}$ and $\hat{\alpha}_{i-1}$ defined in~\eqref{notation:lem}.  
Under Assumption~\ref{asmp:2}, which states $p_{\bX}(\bx) \ge C_1$, we have
\begin{align}
    v_{i-1}&\leq~  \frac{1}{C_1}\EE\Bigg[\int \left(\hat\alpha_{i-1}+\sum_{k=1}^p\sum_{j=1}^{J_i}\hat\beta_{i-1,j}^{(k)}\psi_{jk}\left(x^{(k)}\right) - \sum_{k=1}^p\sum_{j = J_i+1}^\infty \theta_{f,j}^{(k)}\psi_{jk}\left(x^{(k)}\right)\right)^2 p_{\bX}(\bx) d\bx \Bigg]\notag\\
    &=~ \frac{e_{i-1}}{C_1},\label{ineq:inter3}
\end{align}
and also, 
\begin{align}
    \normx{\check g_i-\check f_i}^2 &\leq C_2\int (\check g_i(\bx)-\check f_i(\bx))^2 d\bx \notag\\
    &\leq  C_2\int (g(\bx)-f(\bx))^2 d\bx \notag\\
    &\leq \frac{C_2}{C_1}\normx{g-f}^2, \label{ineq:gi-fi}
\end{align}
Thus, using \eqref{ineq:inter3} and \eqref{ineq:gi-fi} in Lemma \ref{lem:general_bound}, we have
\begin{align*}
    y_{i-1}^2 \geq \frac{e_{i-1}C_1}{2}- 4\left(1+\sqrt{\frac{C_2}{C_1}}\right)^2\normx{g-f}^2 - \frac{4C_1C_2\sum_{k=1}^p\|g_k\|_s^2}{J_i^{2s}}.
\end{align*}
Hence, we have
\begin{align}
&\EE\left[\innerx{\hat{f}_{i-1}-f}{\gamma_ir_i \left(1+\sum_{k=1}^p\sum_{j=1}^{J_i}\psi_{jk}\left(X_i^{(k)}\right)\psi_{jk}\right)}\right]\notag\\
&\leq~-\frac{\gamma_i}{2}C_1e_{i-1}+ 4\gamma_iC_1\left(1+\sqrt{\frac{C_2}{C_1}}\right)^2\normx{g-f}^2 + \gamma_i\frac{4C_1C_2\sum_{k=1}^p\|g_k\|_s^2}{J_i^{2s}}. \label{ineq:term2_2}
\end{align}
By \eqref{equa:fi-f}, \eqref{ineq:term3}, and \eqref{ineq:term2_2}, we notice for any $i\geq 1$,
\begin{align}
    e_i & \leq \left(1-C_1\gamma_i + C_2\left(pM^2J_i+1\right)\gamma_i^2\right)e_{i-1}\notag\\
    &\qquad +C_2\sigma^2(pM^2J_i+1)\gamma_i^2 + 8 \gamma_iC_1C_2 \frac{\sum_{k=1}^p\|g_k\|_s^2}{J_i^{2s}}+8\gamma_iC_1\left(1+\sqrt{\frac{C_2}{C_1}}\right)^2\normx{g-f}^2 \notag\\
    &\leq \left(1-C_1\gamma_i + 2C_2pM^2J_i\gamma_i^2\right)e_{i-1}\notag\\
    &\qquad +2C_2\sigma^2pM^2J_i\gamma_i^2 + 8 \gamma_iC_1C_2 \frac{\sum_{k=1}^p\|g_k\|_s^2}{J_i^{2s}}+8\gamma_iC_1\left(1+\sqrt{\frac{C_2}{C_1}}\right)^2\normx{g-f}^2, \label{ineq:ei}
\end{align}
since $p,M,J_i\geq 1$. 

When $1\leq i \leq p/B$, note that we take $\gamma_i = J_i = 0$ and hence $\hat f_i = 0$.  When $p/B < i \leq p^{1+1/2s}$, since $Bi/p>1$, we know  $Bi/p \leq J_i = \lceil Bi/p \rceil \leq 2Bi/p$. Then, continuing from \eqref{ineq:ei}, we have
\begin{align}
    e_i & \leq \left(1-\frac{C_1A_1}{i}+\frac{4C_2BM^2A_1^2}{i}\right)e_{i-1}+ \frac{\tilde C_1 \normx{g-f}^2 + \tilde C_2 \sigma^2}{i} + \frac{\tilde C_3 p^{1+2s}}{i^{1+2s}}\frac{\sum_{k=1}^p\|g_k\|_s^2}{p},  \label{recur:2}
\end{align}
where $\tilde C_1 = 8A_1C_1(1+\sqrt{C_2/C_1})^2$, $\tilde C_2 = 4C_2BM^2A_1^2$, and $\tilde C_3 = 8A_1C_1C_2/B^{2s}$. Note that $A_1 \geq 2(2s+1)/C_1$ and $B \leq (2s+1)/(4C_2M^2A_1^2)$. Continuing from \eqref{recur:2}, we have for any $p/B < i \leq p^{1+1/2s}$,
\begin{align}
    e_i & \leq \left(1-\frac{2s+1}{i}\right)e_{i-1}+ \frac{\tilde C_1 \normx{g-f}^2 + \tilde C_2\sigma^2}{i} + \frac{\tilde C_3 p^{1+2s}}{i^{1+2s}}\frac{\sum_{k=1}^p\|g_k\|_s^2}{p} \notag\\
    & \leq \left(1-\frac{1}{i}\right)^{2s+1}e_{i-1}+ \frac{\tilde C_1 \normx{g-f}^2 + \tilde C_2\sigma^2}{i} + \frac{\tilde C_3 p^{1+2s}}{i^{1+2s}}\frac{\sum_{k=1}^p\|g_k\|_s^2}{p}\notag\\
    &\leq \left(1-\frac{1}{i}\right)^{2s+1}e_{i-1}+  \frac{ p^{1+2s}}{i^{1+2s}}\left(\tilde C_3\frac{\sum_{k=1}^p\|g_k\|_s^2}{p} +\tilde C_1 \normx{g-f}^2 + \tilde C_2\sigma^2 \right),
\end{align}
where we use Bernoulli's inequality, i.e., $(1+z)^r \geq 1+zr$ for all $ r \geq 1$ and $z \geq -1 $, for the second inequality.\footnote{The interested reader may wish to compare this recursion to \citep[Lemma 1]{chung1954stochastic} in the context of the Robbins–Monro stochastic approximation algorithm.} Note that 
$$
\prod_{j=k+1}^i \left(1-\frac{1}{j}\right)^{2s+1} = \bigg(\frac{k}{i}\bigg)^{2s+1}.
$$
We now invoke Lemma \ref{lem}, taking $i_0 = \lfloor p/B \rfloor+1$ and $i = \lfloor p^{1+1/2s} \rfloor$  in \eqref{lem:inq}, to obtain
\begin{align}
    e_{\lfloor p^{1+1/2s} \rfloor} &\leq \frac{1}{(2B)^{2s+1}}\frac{e_{\lfloor p/B_1 \rfloor}}{p^{(2s+1)/2s}} + 
    \tilde C_3\frac{\sum_{k=1}^p\|g_k\|_s^2}{p} +\tilde C_1 \normx{g-f}^2 + \tilde C_2\sigma^2\notag\\
    &\leq \tilde C_3\frac{\sum_{k=1}^p\|g_k\|_s^2}{p} +\tilde C_1 \normx{g-f}^2 + \frac{1}{(2B)^{2s+1}}\frac{\|f\|^2}{p} +\tilde C_2\sigma^2, \label{stage2}
\end{align}
where in the last equality, we use the fact that $e_{\lfloor p/B \rfloor} = \|f\|^2 $.

When $i > p^{1+1/2s}$, according to \eqref{ineq:ei} and the fact that $Bi^{1/(2s+1)} \leq J_i \leq 2Bi^{1/(2s+1)}$, we have
\begin{align}
    e_i & \leq \left(1-\frac{C_1A_2}{i}+\frac{4pC_2BM^2A_2^2}{i^{1+\frac{2s}{2s+1}}}\right)e_{i-1}+ \frac{\tilde C_4\left(p\sigma^2+\sum_{k=1}^p\|g_k\|_s^2\right)}{i^{1+\frac{2s}{2s+1}}} + \frac{\tilde C_5 \normx{g-f}^2}{i},  \label{recur:1}
\end{align}
where $\tilde C_4 =  \max\{4C_2BM^2A_2^2, 8A_2C_1C_2/B^{2s}\}$ and $\tilde C_5 = 8A_2C_1\big(1+\sqrt{C_2/C_1}\big)^2$. Note that when $A_2 \geq 2/C_1$ and $B \leq 1/(4C_2M^2A_2^2)$, continuing from \eqref{recur:1}, we have for any $i > p^{1+1/2s}$,
\begin{align*}
    e_i \leq \Big(1-\frac{1}{i}\Big)e_{i-1} + \frac{\tilde C_4\left(p\sigma^2+\sum_{k=1}^p\|g_k\|_s^2\right)}{i^{1+\frac{2s}{2s+1}}}+\frac{\tilde C_5 \normx{g-f}^2}{i}.
\end{align*}
We use Lemma \ref{lem} again, this time taking $i_0 = \lfloor p^{1+1/2s}\rfloor+1$ and $i=n$. We obtain
\begin{align}
    e_n &\leq e_{\lfloor p^{1+1/2s}\rfloor}\frac{p^{1+1/2s}}{n} + \frac{n - p^{1+1/2s}-1}{n}\tilde C_5 \normx{g-f}^2 + \sum_{i=\lfloor p^{1+1/2s}\rfloor+1}^{n-1}\frac{\tilde C_4\left(p\sigma^2+\sum_{k=1}^p\|g_k\|_s^2\right)}{ni^{\frac{2s}{2s+1}}} \notag\\
    &\leq  pn^{-\frac{2s}{2s+1}}C_0^{-\frac{1}{2s+1}}\left( \tilde C_3\frac{\sum_{k=1}^p\|g_k\|_s^2}{p} +\tilde C_1 \normx{g-f}^2 + \frac{1}{(2B)^{2s+1}}\frac{\|f\|^2}{p} +\tilde C_2\sigma^2\right)\notag\\
    & \qquad +\tilde C_5 \normx{g-f}^2 + \tilde C_4\left(p\sigma^2+\sum_{k=1}^p\|g_k\|_s^2\right)(2s+1)n^{-\frac{2s}{2s+1}}, \label{en}
\end{align}
where for the last inequality we use \eqref{stage2} and the fact that $\frac{1}{n}\sum_{i=1}^n i^{-\frac{2s}{2s+1}} \leq (2s+1)n^{-\frac{2s}{2s+1}}$.
The inequality \eqref{en} holds for $g = \tau + \sum_{k=1}^p g_k$ with any $g_k \in \mathcal{H}_1(s)$ for  $k=1,2,\dots,p$ and any $\tau \in \RR$.

So we conclude that there exists a constant $C$ such that
\begin{equation*}
    \EE \left[\norm{\hat{f}_n - f}^2 \right] \leq C\inf_{g\in \mathcal{H}_p(s)} \Bigg\{ \|g-f\|^2 + \left(p\sigma^2+\|f\|^2+\sum_{k=1}^p\|g_k\|^2_s \right) n^{-\frac{2s}{2s+1}}\Bigg\}.
\end{equation*}
Here $C$ depends on $C_0,C_1,C_2,M$, and $s$.
\end{proof}
\begin{proof}[Proof of Theorem \ref{thm:fixed}]
The proof of Theorem \ref{thm:fixed} follows a similar argument as that of Theorem \ref{theorem:l2loss}. 

When $i < 1/B^{2s+1}$, we note that $J_i = 0$. Similar to \eqref{ineq:term3}, we have $\mathbb{E}[\gamma_i^2r_i^2]\leq \gamma_i^2(\sigma^2+e_{i-1}^2)$. Therefore, using the second equality in~\eqref{ineq:inter1} and taking expectations in~\eqref{equa:fi-f} with $J_i = 0$, we obtain
\begin{align*}
    e_i &\leq (1+\gamma_i^2)e_{i-1} + \gamma_i^2\sigma^2 - 2\gamma_i \mathbb{E}\bigg[\left|\innerx{\hat{f}_{i-1}-f}{1}\right|^2\bigg] \\
    & \leq (1+\gamma_i^2)e_{i-1} + \gamma_i^2\sigma^2.
\end{align*}
Note that we have  
\begin{align*}
    \prod_{j=k+1}^i\Bigg(1+\frac{A^2}{j^2}\Bigg) &= \exp\Bigg\{\sum_{j=k+1}^i\log\Bigg(1+\frac{A^2}{j^2}\Bigg)\Bigg\}\\
    &\leq \exp\Bigg\{\sum_{j=k+1}^i\frac{A^2}{j^2}\Bigg\}\\
    &\leq \exp\Bigg\{\sum_{j=k+1}^iA^2\left(\frac{1}{j-1} - \frac{1}{j}\right)\Bigg\} \leq e^{A^2/k},
\end{align*}
and 
\begin{align*}
    \sum_{k=1}^{\lfloor 1/B^{2s+1} \rfloor-1} \frac{1}{k^2}e^{A^2/k} \leq \int_1^{\lfloor 1/B^{2s+1} \rfloor} \frac{1}{x^2}e^{A^2/x} dx \leq \frac{e^{A^2}}{A^2}.
\end{align*}
Thus, using Lemma \ref{lem:inq} by taking $i_0=1$ and $i = \lfloor 1/B^{2s+1} \rfloor$, we know
\begin{align}
    e_{\lfloor 1/B^{2s+1} \rfloor} \leq \tilde C_1\|f\|^2 + \tilde C_2 \sigma^2, \label{stage2:fixed}
\end{align}
where $\tilde C_1 = e^{1.5A^2}$, $\tilde C_2 = e^{A^2}/A^2 + A^2/\lfloor 1/B^{2s+1} \rfloor^2$, and $\|f\|^2$ is a constant that depends implicitly on $p$ since $f$ has $p$ component functions. When $i > \lfloor 1/B^{2s+1} \rfloor$, we know $Bi^{1/(2s+1)}/2 \leq J_i \leq Bi^{1/(2s+1)}$. Continuing from \eqref{ineq:ei}, we have
\begin{align*}
    e_i & \leq \left(1-\frac{C_1A}{i}+\frac{2pC_2BM^2A^2}{i^{1+\frac{2s}{2s+1}}}\right)e_{i-1}+ \frac{\tilde C_3\left(p\sigma^2+\sum_{k=1}^p\|g_k\|_s^2\right)}{i^{1+\frac{2s}{2s+1}}} + \frac{\tilde C_4 \normx{g-f}^2}{i}\\
    & \leq \left(1-\frac{1}{i}\right)e_{i-1}+ \frac{\tilde C_3\left(p\sigma^2+\sum_{k=1}^p\|g_k\|_s^2\right)}{i^{1+\frac{2s}{2s+1}}} + \frac{\tilde C_4 \normx{g-f}^2}{i},
\end{align*}
where $\tilde C_3 =  \max\{2C_2BM^2A^2, 2^{2s+3}AC_1C_2/B^{2s}\}$ and $\tilde C_4 = 8AC_1\big(1+\sqrt{C_2/C_1}\big)^2$, and $A$ and $B$ satisfy  $A \geq \frac{2}{C_1}$ and $B \leq \frac{1}{2pC_2M^2A^2}$, respectively. We use Lemma \ref{lem} again by taking $i_0 = \lfloor p^{1+1/2s}\rfloor+1$ and $i=n$. We obtain
\begin{align}
    e_n &\leq e_{\lfloor p^{1+1/2s}\rfloor}\frac{p^{1+1/2s}}{n} + \frac{n - p^{1+1/2s}-1}{n}\tilde C_4 \normx{g-f}^2 + \sum_{i=\lfloor p^{1+1/2s}\rfloor+1}^{n-1}\frac{\tilde C_3\left(p\sigma^2+\sum_{k=1}^p\|g_k\|_s^2\right)}{ni^{\frac{2s}{2s+1}}} \notag\\
    &\leq  \frac{p^{1+1/2s}}{n}\Big(C_1\|f\|^2 + \tilde C_2 \sigma^2\Big) +\tilde C_4 \normx{g-f}^2 + \tilde C_3\left(p\sigma^2+\sum_{k=1}^p\|g_k\|_s^2\right)(2s+1)n^{-\frac{2s}{2s+1}}, \label{en:fixed}
\end{align}
where for the last inequality we use \eqref{stage2:fixed} and the fact that $\frac{1}{n}\sum_{i=1}^n i^{-\frac{2s}{2s+1}} \leq (2s+1)n^{-\frac{2s}{2s+1}}$.
The inequality \eqref{en:fixed} holds for $g = \tau + \sum_{k=1}^p g_k$ with any $g_k \in \mathcal{H}_1(s)$ for  $k=1,2,\dots,p$ and any $\tau \in \RR$.

So we conclude that there exists a constant $C$ such that
\begin{equation*}
    \EE \left[\norm{\hat{f}_n - f}^2 \right] \leq C\inf_{g\in \mathcal{H}_p(s)} \Bigg\{ \|g-f\|^2 + \left(\sigma^2+\|f\|^2+\sum_{k=1}^p\|g_k\|^2_s \right) n^{-\frac{2s}{2s+1}}\Bigg\}.
\end{equation*}
Here $C$ depends on $p,C_1,C_2,M$, and $s$.
\end{proof}

\section*{Funding}
Klusowski was supported in part by the National Science Foundation through CAREER DMS-2239448.

\section*{Acknowledgments}
The authors would like to thank the Associate Editor and anonymous referees for their helpful comments and suggestions that improved the presentation of the paper.

\section*{Supplementary Material}

The supplementary material \citep{chen2025supplement} contains the full proof of Theorem \ref{thm:polynomial}.

\bibliographystyle{plainnat}
\bibliography{ref}

\newpage

\setcounter{page}{1}
\renewcommand{\thepage}{\arabic{page}}
\setcounter{section}{0}
\setcounter{equation}{0}
\setcounter{theorem}{0}

\renewcommand{\thesection}{S\arabic{section}}
\renewcommand{\theequation}{S\arabic{equation}}
\renewcommand{\thetheorem}{S\arabic{theorem}}

\begin{center}
\Large\textbf{Supplement to ``Stochastic Gradient Descent for Nonparametric Additive Regression''}
\end{center}

\vspace{1em}

\noindent In this supplement, we provide the proof of Theorem \ref{thm:polynomial}.

\begin{proof}[Proof of Theorem \ref{thm:polynomial}]
Note that in Theorem \ref{thm:polynomial}, we assume $f$ belongs to a multivariate Sobolev ellipsoid $\tilde{\mathcal{W}}_p(s,Q, {\psi_j})$. Using the same reasoning as in the proof of Theorem \ref{theorem:l2loss}, we can modify \eqref{ineq:general_recursion} to:
\begin{align}
    e_i \leq e_{i-1} - 2\gamma_iy_{i-1}^2 + \gamma_i^2\left(\sigma^2 + e_{i-1}\right)C_2J_iM^2. \label{ineq:recur1}
\end{align}
We begin by providing a simple increasing bound for $e_i$, which will be used to establish the desired polynomial decay rate. Specifically, We  show that 
\begin{align}
    e_i \leq 2C_2Q^2\log(i+2). \label{ei_logrithmic}
\end{align}
We use the induction to prove this results. It is straightforward to see $\eqref{ei_logrithmic}$ holds for $i=0$ since $\|f\| \leq C_2Q^2$. Suppose \eqref{ei_logrithmic} holds for $i-1$. Note that 
\begin{align}
    e_i &\leq e_{i-1} + \gamma_i^2\left(\sigma^2 + e_{i-1}\right)C_2J_iM^2\\
    & \leq 2C_2Q^2\log(i+1) + 2A^2C_2M^2i^{-\frac{8s+1}{6s+1}}(\sigma^2 + 2C_2Q^2\log(i+1))\\
    & \leq 2C_2Q^2\log(i+1) + \frac{2C_2Q^2}{3}\log(i+1)i^{-\frac{8s+1}{6s+1}}\\
    & \leq 2C_2Q^2\log(i+1) + \frac{2C_2Q^2}{3}\log(i+1)i^{-\frac{5}{4}}\\
    & \leq 2C_2Q^2\log(i+1) + \frac{2C_2Q^2}{i+2} \leq 2C_2Q^2\log(i+2),
\end{align}
where for the third inequality we use the selection of $A$ such that $A^2 \leq 1/14C_2M^2\max\{\sigma^2, 1\}$. This proves \eqref{equa:fi-f}.

Given that $f$ belongs to $\tilde{\mathcal{W}}_p(s,Q, {\psi_j})$, we can set $g = f$ in Lemma \ref{lem:general_bound} and leverage the fact that $\|f\|_s^2 < Q^2$. This adaptation modifies \eqref{ineq:general_bound} to:
\begin{align}
v_{i-1}y_{i-1}^2
\geq~\frac{e_{i-1}^2}{2} - \frac{4e_{i-1}C_2Q^2}{J_i^{2s}}. \label{ineq:general_bound2}
\end{align}
Next, we will establish an upper bound for $v_i$. Note that
\begin{align}
     v_i &= \EE\Bigg[\int \left(\hat{f}_{i}(\boldsymbol{x})-f(\boldsymbol{x})\right)^2 d\boldsymbol{x}\Bigg] \notag\\
     & = \EE\Bigg[\int \left(\hat{f}_{i-1}(\bx) - f(\boldsymbol{x})+\gamma_ir_i\Bigg(\sum_{j=1}^{J_i}\psi_{j}\left(\bX_i\right)\psi_{j}(\bx)\Bigg)\right)^2 d\boldsymbol{x}\Bigg]\notag\\
     &=v_{i-1}+2\gamma_i \EE\Bigg[r_i \int \Bigg(\hat{f}_{i-1}(\bx) - f(\boldsymbol{x})\Bigg)\Bigg(\sum_{j=1}^{J_i}\psi_{j}\left(\bX_i\right)\psi_{j}(\bx) \Bigg)d\boldsymbol{x}  \Bigg] + \gamma_i^2 \EE \Bigg[r_i^2\sum_{j=1}^{J_i}\psi_j^2(\bX_i) \Bigg]. \label{ineq:v1}
\end{align}
Using the same reasoning as employed in deriving the inequality \eqref{ineq:term3}, we know 
\begin{align}
   \gamma_i^2 \EE \Bigg[r_i^2\sum_{j=1}^{J_i}\psi_j^2(\bX_i) \Bigg] \leq \gamma_i^2(\sigma^2 + e_{i-1})J_iM^2. \label{ineq:v2} 
\end{align}
In addition, by the Cauchy-Schwarz inequality,
\begin{align}
    &\EE\Bigg[r_i \int \Bigg(\hat{f}_{i-1}(\bx) - f(\boldsymbol{x})\Bigg)\Bigg(\sum_{j=1}^{J_i}\psi_{j}\left(\bX_i\right)\psi_{j}(\bx) \Bigg)d\boldsymbol{x}  \Bigg] \notag\\
    &=~ \EE \Bigg[ \sum_{j=1}^{J_i}\int \Big(\hat{f}_{i-1}(\bx) - f(\boldsymbol{x})\Big)\psi_j(\bx)p_{\bX}(\bx) d\bx \int \Big(\hat{f}_{i-1}(\bx) - f(\boldsymbol{x})\Big)\psi_j(\bx) d\bx  \Bigg]\notag\\
    &\leq~ y_{i-1}\sqrt{b_{i-1}}. \label{ineq:v3}
\end{align}
Thus, according to \eqref{ineq:v1}, \eqref{ineq:v2} and \eqref{ineq:v3}, we have
\begin{align}
    v_i \leq v_{i-1} + 2\gamma_iy_{i-1}\sqrt{v_{i-1}} + M^2\gamma_i^2J_i(\sigma^2 + e_{i-1}).
\end{align}
Let $\tilde v_i = \max\{v_i, v_{i-1},\dots,v_0\}$. We use induction to show
\begin{align}
    \tilde v_i \leq \tilde v_{i-1} + 2\gamma_iy_{i-1}\sqrt{\tilde v_{i-1}} + M^2\gamma_i^2J_i(\sigma^2 + e_{i-1}). \label{ineq: recur2}
\end{align}
Note that \eqref{ineq: recur2} holds for $i=0$. Suppose \eqref{ineq: recur2} holds for $i=k-1$. Then
\begin{align}
    \tilde v_{k-1} \leq \tilde v_{k-2} + 2\gamma_iy_{i-1}\sqrt{\tilde v_{k-2}} + M^2\gamma_i^2J_i(\sigma^2 + e_{i-1}) \leq \tilde v_{k-1} + 2\gamma_iy_{i-1}\sqrt{\tilde v_{k-1}} + M^2\gamma_i^2J_i(\sigma^2 + e_{i-1}), 
\end{align}
since $\{\tilde v_i\}$ is non-decreasing. Furthermore,
\begin{align*}
    v_k \leq  v_{k-1} + 2\gamma_iy_{i-1}\sqrt{v_{k-1}} + M^2\gamma_i^2J_i(\sigma^2 + e_{i-1})   \leq \tilde v_{k-1} + 2\gamma_iy_{i-1}\sqrt{\tilde v_{k-1}} + M^2\gamma_i^2J_i(\sigma^2 + e_{i-1}).
\end{align*}
Since $\tilde v_k = \max\{v_k, \tilde v_{k-1}\}$, we know \eqref{ineq: recur2} also holds for $i=k$. In addition, according to \eqref{ineq:general_bound2}, we have
\begin{align}
\tilde v_{i-1}y_{i-1}^2
\geq~\frac{e_{i-1}^2}{2} - \frac{4e_{i-1}C_2Q^2}{J_i^{2s}}. \label{ineq:tilde_v}
\end{align}
In the following analysis, we will utilize \eqref{ineq:recur1}, \eqref{ineq: recur2}, and \eqref{ineq:tilde_v} to demonstrate that $e_n = O\left(n^{-\frac{s-1/2}{6s+1}}\right)$. Let $u_i = e_i/\tilde v_i$. We will first show $u_n = O\left(n^{-\frac{2s}{6s+1}}\right)$. By \eqref{ineq:recur1} and \eqref{ineq:tilde_v}, we know 
\begin{align*}
    e_{i} \leq (1+C_2M^2\gamma_i^2J_i)e_{i-1} - 2\gamma_i\Bigg( \frac{e_{i-1}^2}{2\tilde v_{i-1}} - \frac{4e_{i-1}C_2Q^2}{J_i^{2s}\tilde v_{i-1}}\Bigg) + \sigma^2C_2M^2\gamma_i^2J_i.
\end{align*}
Since $\tilde v_i$ is non-decreasing, it follows that
\begin{align}
    u_i &\leq  (1+C_2M^2\gamma_i^2J_i)\frac{e_{i-1}}{\tilde v_{i-1}} - \gamma_i \frac{e_{i-1}^2}{\tilde v_{i-1}^2} + \frac{8\gamma_ie_{i-1}C_2Q^2}{J_i^{2s}\tilde v_{i-1}^2} + \frac{\sigma^2C_2M^2\gamma_i^2J_i}{\tilde v_{i-1}}\notag\\
    &\leq u_{i-1}\Bigg(1+ \frac{8AC_2Q^2}{\tilde v_0 }i^{-1} + 2C_2A^2M^2i^{-\frac{8s+1}{6s+1}} - Au_{i-1}i^{-\frac{4s+1}{6s+1}}\Bigg) + \frac{C_2A^2M^2\sigma^2}{\tilde v_0 } i^{-\frac{8s+1}{6s+1}}, \label{ineq:u1} 
\end{align}
where for the second inequality, we specify $\gamma_i = Ai^{-(4s+1)/(6s+1)}$ and $J_i = \lceil i^{\frac{1}{6s+1}} \rceil$. Also, we use the facts that $i^{\frac{1}{6s+1}} \leq J_i \leq 2i^{\frac{1}{6s+1}}$ and $\tilde v_{i-1}\geq \tilde v_0$. From \eqref{ineq:u1}, we have
\begin{align}
    i^{\frac{4s+1}{6s+1}}(u_i - u_{i-1}) &\leq i^{-\frac{2s}{6s+1}}\bigg(\frac{8AC_2Q^2}{\tilde v_0 } + 2C_2A^2M^2 \bigg)u_{i-1} - Au_{i-1}^2 + \frac{C_2A^2M^2\sigma^2}{\tilde v_0} i^{-\frac{4s}{6s+1}}\label{ineq:u_recur}
\end{align}
which implies
\begin{align}
    \sum_{i=1}^ni^{\frac{4s+1}{6s+1}}(u_i - u_{i-1}) \leq \bigg(&\frac{8AC_2Q^2}{\tilde v_0 } + 2C_2A^2M^2 \bigg)u_0 +  \sum_{i=2}^n (i-1)^{-\frac{2s}{6s+1}}\bigg(\frac{8AC_2Q^2}{\tilde v_0 } + 2C_2A^2M^2 \bigg)u_{i-1} \\
    &- Au_{i-1}^2 + \frac{C_2A^2M^2\sigma^2}{\tilde v_0} i^{-\frac{4s}.{6s+1}}\label{ineq:u_recur_sum}
\end{align}
Note that 
\begin{align}
    \sum_{i=1}^n i^{\frac{4s+1}{6s+1}}(u_i - u_{i-1}) &= n^{\frac{4s+1}{6s+1}}u_n - \sum_{i=1}^n \bigg(i^{\frac{4s+1}{6s+1}} - (i-1)^{\frac{4s+1}{6s+1}}\bigg)u_{i-1}\notag\\
    & \geq n^{\frac{4s+1}{6s+1}}u_n - u_0 - \sum_{i=2}^n \frac{4s+1}{6s+1}(i-1)^{-\frac{2s}{6s+1}}  u_{i-1} \label{ineq:u_sum} 
\end{align}
where for the inequality, we use the fact that
\begin{align*}
    i^{\frac{4s+1}{6s+1}} - (i-1)^{\frac{4s+1}{6s+1}} \leq  \frac{4s+1}{6s+1}(i-1)^{-\frac{2s}{6s+1}}.
\end{align*}
According to \eqref{ineq:u_recur_sum} and  \eqref{ineq:u_sum}, we know
\begin{align}
    n^{\frac{4s+1}{6s+1}}u_n &\leq \bigg(\frac{8C_2AQ^2}{\tilde v_0} + 2C_2A^2M^2 + 1 \bigg)u_0 +  \sum_{i=2}^n (i-1)^{-\frac{2s}{6s+1}}\bigg(\frac{8C_2AQ^2}{\tilde v_0} + 2C_2A^2M^2 + \frac{4s+1}{6s+1} \bigg)u_{i-1} \\&~~~~-A\sum_{i=1}^n u_{i-1}^2 + \sum_{i=1}^n\frac{C_2A^2M^2\sigma^2}{\tilde v_0} i^{-\frac{4s}{6s+1}}\\
    &\leq \frac{1}{4A} \bigg(\frac{8C_2AQ^2}{\tilde v_0} + 2C_2A^2M^2+\frac{4s+1}{6s+1} \bigg)^2   \sum_{i=2}^n (i-1)^{-\frac{4s}{6s+1}} + \frac{C_2A^2M^2\sigma^2}{\tilde v_0}\sum_{i=1}^n i^{-\frac{4s}{6s+1}} \\&~~~~+ \frac{1}{4A} \bigg(\frac{8C_2AQ^2}{\tilde v_0} + 2C_2A^2M^2+1 \bigg)^2\\
    &\leq \tilde C_1 n^{\frac{2s+1}{6s+1}},
\end{align}
where 
\begin{align*}
    \tilde C_1 &= \frac{6s+1}{2s+1}\left(\frac{1}{4A} \bigg(\frac{8C_2AQ^2}{\tilde v_0} + 2C_2A^2M^2+\frac{4s+1}{6s+1} \bigg)^2 + \frac{C_2A^2M^2\sigma^2}{\tilde v_0} \right).
\end{align*}
Here we use the inequalities $$\sum_{i=1}^n i^{-\frac{4s}{6s+1}} \leq \int_0^n x^{-\frac{4s}{6s+1}}~dx = \frac{6s+1}{2s+1}n^{\frac{2s+1}{6s+1}},$$ and
$$
\frac{1}{4A} \bigg(\frac{8C_2AQ^2}{\tilde v_0} + 2C_2A^2M^2+1 \bigg)^2 \leq \tilde C_1.
$$
Thus, $u_n \leq  \tilde C_1n^{-\frac{2s}{6s+1}}$. Let $r_i = e_i\tilde v_i$. Next, we will show $r_n = O\left(n^{\frac{1}{4}}\right)$. Specifically, we will use induction to show $r_i \leq Bi^{\frac{1}{4}}$ for any $i \geq 1$ where $B = 4C_2Q^2(Q^2+1)$. From \eqref{ei_logrithmic}, we know $e_1 \leq 4C_2Q^2$. Note that, by the Cauchy-Schwarz inequality, we have
\begin{align}
    y_{i-1}^2 = \sum_{j=1}^{J_i}\mathbb{E}\left[\left|\innerx{\hat{f}_{i-1}-f}{\psi_{j}}\right|^2\right] \leq C_2J_ie_{i-1}.
\end{align}
Thus,  
\begin{align}
    v_1 &\leq v_0 + 2\gamma_1\sqrt{C_2J_1e_0v_0} + M^2\gamma_1^2J_1(\sigma^2 + e_0)\\
    & \leq Q^2 + 4AC_2Q^2 + A^2M^2\big(\sigma^2 + C_2Q^2\big)\\ &\leq Q^2+1
\end{align}
provided that $4AC_2Q^2 + A^2M^2\big(\sigma^2 + C_2Q^2\big) \leq 1$. Here, we use the facts that $v_0 \leq  Q^2$ and $e_0 = \|f\| \leq  C_2Q^2$ in the  inequalities above. Therefore, the base case holds since $r_1 = e_1\tilde v_1 \leq B$.   

Suppose that $r_i \leq Bi^{\frac{1}{4}}$ for each $1\leq i \leq k$. It follows from \eqref{ineq: recur2} that for each $1\leq i\leq k$
\begin{align}
    \tilde v_i &\leq \tilde v_{i-1} + 2\sqrt{C_2\gamma_i^2J_{i}e_{i-1}\tilde v_{i-1}} +  M^2\gamma_i^2J_i(\sigma^2 + e_{i-1}) \notag\\
    &\leq \tilde v_{i-1} + \tilde C_2i^{-\frac{1}{2}-\frac{s/4 - 1/8}{6s+1}}, \label{ineq: temp1}
\end{align}
where 
\begin{align*}
    \tilde C_2 = 2\sqrt{2C_2B} + 1.
\end{align*}
Here we use the fact that $$M^2\gamma_i^2J_i(\sigma^2 + e_{i-1}) \leq 2A^2M^2i^{-\frac{8s+1}{6s+1}}\big(\sigma^2 + 2C_2Q^2\log(i+1)\big) \leq i^{-\frac{1}{2}-\frac{s/4 - 1/8}{6s+1}},$$
when $A^2 \leq 1/14M^2\max\{\sigma^2, 2C_2Q^2\}$.  Sum over $i=1$ to $k$ in \eqref{ineq: temp1}, we know 
\begin{align}
    \tilde v_k \leq \tilde C_3 k^{\frac{1}{2}-\frac{s/4 - 1/8}{6s+1}}, \label{ineq:vk}
\end{align}
where $\tilde C_3 = 3\tilde C_2 + Q^2 = 6\sqrt{2C_2B} + 3 + Q^2$ since $\tilde v_0 = v_0 \leq Q^2$. On the other hand, by continuing to apply \eqref{ineq: recur2} and using the inequality $\frac{2y_{i-1}}{\sqrt{\tilde v_{i-1}}} \leq \frac{2y_{i-1}^2}{e_{i-1}} + \frac{e_{i-1}}{2\tilde v_{i-1}}$, we obtain
\begin{align}
    \tilde v_i \leq \tilde v_{i-1}\left(1+ \frac{2\gamma_iy_{i-1}^2}{e_{i-1}}\right) + \frac{\gamma_i e_{i-1}}{2}  + M^2\gamma_i^2J_i(\sigma^2 + e_{i-1}). \label{ineq:temp2}
\end{align}
Taking $i = k+1$ in \eqref{ineq:recur1} and  \eqref{ineq:temp2}, we have
\begin{align}
    e_{k+1} &\leq e_k\left(1- \frac{2\gamma_{k+1}y_k^2}{e_k} \right) + 2A^2C_2M^2k^{-\frac{8s+1}{6s+1}}\big(\sigma^2 + 2C_2Q^2\log(k+2)\big), \label{ineq:temp3}\\
    \tilde v_{k+1} &\leq \tilde v_{k}\left(1+ \frac{2\gamma_{k+1}y_{k}^2}{e_{k}}\right) + \frac{\gamma_{k+1} e_{k}}{2}  + 2A^2M^2k^{-\frac{8s+1}{6s+1}}\big(\sigma^2 + 2C_2Q^2\log(k+2)\big). \label{ineq:temp4}
\end{align}
By multiplying \eqref{ineq:temp3} 
and \eqref{ineq:temp4} together, and using the facts that $e_k^2 = u_kr_k$, $C_2 \geq 1$, and the selection of $A^2$ such that 
$$
2A^2M^2k^{-\frac{8s+1}{6s+1}}\big(\sigma^2 + 2C_2Q^2\log(k+2)\big) \leq C_2Q^2\log(k+2),
$$
we obtain
\begin{align}
    r_{k+1} &\leq r_k + \frac{1}{2} \gamma_{k+1}u_kr_k +  2A^2C_2M^2k^{-\frac{8s+1}{6s+1}}\big(\sigma^2 + 2C_2Q^2\log(k+2)\big)\tilde v_{k}\left(1+ \frac{2\gamma_{k+1}y_{k}^2}{e_{k}}\right)\\
    &~~~~+ 2A^2C_2M^2k^{-\frac{8s+1}{6s+1}}\big(\sigma^2 + 2C_2Q^2\log(k+2)\big)\bigg(3C_2Q^2\log(k+2) + AC_2Q^2i^{-\frac{4s+1}{6s+1}}\bigg).\label{recursion_rk_1}
\end{align}
Note that $2\gamma_{k+1}y_k^2/e_k \leq 2C_2\gamma_{k+1}J_{k+1} \leq 4AC_2(k+1)^{-4s/(6s+1)} \leq 1$ when $A \leq 1/4C_2$. Continuing from \eqref{recursion_rk_1}, and applying \eqref{ineq:vk}, the inequality $u_k \leq \tilde C_1 k^{-2s/(6s+1)} \leq 2^{1/3}\tilde C_1(k+1)^{-2s/(6s+1)}$, and the induction hypothesis $r_k \leq Bk^{\frac{1}{4}}$, we can get 
\begin{align}
    r_{k+1} &\leq Bk^{\frac{1}{4}}+ 2^{-1+\frac{1}{3}}A\tilde C_1B(k+1)^{-\frac{3}{4}} \\
    &~~~~+ 4A^2C_2M^2\big(6\sqrt{2C_2B} + 3 + Q^2\big)k^{\frac{-21s/4 - 3/8}{6s+1}}\big(\sigma^2 + 2C_2Q^2\log(k+2)\big)\\
    &~~~~+ 2A^2C_2M^2k^{-\frac{8s+1}{6s+1}}\big(\sigma^2 + 2C_2Q^2\log(k+2)\big)\bigg(3C_2Q^2\log(k+2) + AC_2Q^2i^{-\frac{4s+1}{6s+1}}\bigg)\label{recursion_rk_2}
\end{align}
When $A$ is sufficiently small, for each $k \geq 1$, we have,
\begin{align}
    &4A^2C_2M^2\big(6\sqrt{2C_2B} + 3 + Q^2\big)k^{\frac{-21s/4 - 3/8}{6s+1}}\big(\sigma^2 + 2C_2Q^2\log(k+2)\big)\\ 
    &\leq~ \frac{1}{3}\left(\frac{1}{4} - \frac{2^{\frac{1}{3}}(4s+1)^2}{8(2s+1)(6s+1)}\right)B(k+1)^{-\frac{3}{4}},\label{rk_1}
\end{align}
because, for $s > 1/2$, we know that $(21s/4+3/8)/(6s+1) > 3/4$ and 
$$
\frac{1}{4} - \frac{2^{\frac{1}{3}}(4s+1)^2}{8(2s+1)(6s+1)} > 0.
$$
In addition, when $A$ is sufficiently small, for each $k \geq 1$,
\begin{align}
    &2A^2C_2M^2k^{-\frac{8s+1}{6s+1}}\big(\sigma^2 + 2C_2Q^2\log(k+2)\big)\bigg(3C_2Q^2\log(k+2) + AC_2Q^2i^{-\frac{4s+1}{6s+1}}\bigg)\\
    &\leq~ \frac{1}{3}\left(\frac{1}{4} - \frac{2^{\frac{1}{3}}(4s+1)^2}{8(2s+1)(6s+1)}\right)B(k+1)^{-\frac{3}{4}}.\label{rk_2}
\end{align}
Furthermore, when $A$ is sufficiently small, for each $k \geq 1$, it holds that
\begin{align}
    &2^{-1+\frac{1}{3}}A\tilde C_1B(k+1)^{-\frac{3}{4}} \\
    &=~ \frac{2^{\frac{1}{3}}(6s+1)}{2(2s+1)}\left(\frac{1}{4} \bigg(\frac{8C_2AQ^2}{\tilde v_0} + 2C_2A^2M^2+\frac{4s+1}{6s+1} \bigg)^2 + \frac{C_2A^3M^2\sigma^2}{\tilde v_0} \right)B(k+1)^{-\frac{3}{4}}\\
    &\leq~ \frac{2^{\frac{1}{3}}(4s+1)^2}{8(2s+1)(6s+1)}B(k+1)^{-\frac{3}{4}} + \frac{1}{3}\left(\frac{1}{4} - \frac{2^{\frac{1}{3}}(4s+1)^2}{8(2s+1)(6s+1)}\right)B(k+1)^{-\frac{3}{4}}. \label{rk_3} 
\end{align}
Plugging in \eqref{rk_1}, \eqref{rk_2}, and \eqref{rk_3} into \eqref{recursion_rk_2}, we have
\begin{align}
    r_{k+1} \leq Bk^{\frac{1}{4}} + \frac{1}{4}B(k+1)^{-\frac{3}{4}} \leq B(k+1)^{\frac{1}{4}}.
\end{align}
This concludes the induction. Therefore, we obtain
\begin{equation}
e_n = \sqrt{u_nr_n}  = O\Big(n^{\frac{-s/4+1/8}{6s+1}}\Big),
\end{equation}
as desired.
\end{proof}

\end{document}